\documentclass[journal]{IEEEtran}

\usepackage[utf8]{inputenc} 
\usepackage[T1]{fontenc}    
\usepackage{hyperref}       
\usepackage{url}            
\usepackage{booktabs}
\usepackage{graphicx}
\usepackage{mathrsfs}
\usepackage{amssymb}
\usepackage{amsfonts}
\usepackage{fdsymbol}
\usepackage{wrapfig}
\usepackage{caption}
\usepackage[ruled,linesnumbered,lined]{algorithm2e}
\usepackage{subfigure}
\usepackage{multirow}
\usepackage{amsmath,bm}
\usepackage{graphicx}
\usepackage{amsfonts}       
\usepackage{nicefrac}       
\usepackage{microtype}      
\usepackage{hyperref}
\usepackage{amsthm}
\usepackage[bottom]{footmisc}

\newtheorem{myTheo}{Theorem}
\newtheorem{myPro}{Problem}
\newtheorem{myAss}{Assumption}
\usepackage{url}

\usepackage{amsmath}
\usepackage{amssymb}
\usepackage{color}
\usepackage{bm}

\ifCLASSINFOpdf
\else
\fi
\usepackage{amsmath}
\begin{document}
%
\title{Debiased Graph Neural Networks with\\Agnostic Label Selection Bias}
%
%
%

\author{Shaohua Fan,
Xiao Wang,~\IEEEmembership{Member,~IEEE,} \IEEEauthorblockN{Chuan Shi\thanks{\IEEEauthorrefmark{2} Corresponding author.}\IEEEauthorrefmark{2}},~\IEEEmembership{Member,~IEEE,}
        Kun Kuang,
        Nian Liu, Bai Wang

\thanks{S. Fan, N. Liu and B. Wang are with the Department
of Computer Science, Beijing University of Posts and Telecommunications,
Beijing, China.
E-mail: \{fanshaohua, nianliu, wangbai\}@bupt.edu.cn}
\thanks{X. Wang and C. Shi are with the Key Laboratory of Trustworthy Distributed Computing and Service (MoE), Beijing University of Posts and Telecommunications, Beijing, China and the Peng Cheng Laboratory, Shenzhen, China.
E-mail: \{xiaowang,shichuan\}@bupt.edu.cn}
\thanks{K. Kuang is with the College of Computer Science and Technology,  Zhejiang University, Zhejiang, China.
E-mail: kunkuang@zju.edu.cn}
\thanks{Manuscript received 05 Feb. 2021; revised 02 Sep. 2021; accepted 29 Dec 2021.}}

%
%

\markboth{Journal of \LaTeX\ Class Files,~Vol.~14, No.~8, August~2015}%
{IEEE TRANSACTIONS ON NEURAL NETWORKS AND LEARNING SYSTEMS}
%
%



\maketitle

\begin{abstract}
Most existing Graph Neural Networks (GNNs) are proposed without considering the selection bias in data, i.e., the inconsistent distribution between the training set with test set. In reality, the test data is not even available during the training process, making selection bias agnostic. Training GNNs with biased selected nodes leads to significant parameter estimation bias and greatly impacts the generalization ability on test nodes. In this paper, we first present an experimental investigation, which clearly shows that the selection bias drastically hinders the generalization ability of GNNs, and theoretically prove that the selection bias will cause the biased estimation on GNN parameters. Then to remove the bias in GNN estimation, we propose a novel Debiased Graph Neural Networks (DGNN) with a differentiated decorrelation regularizer. The differentiated decorrelation regularizer estimates a sample weight for each labeled node such that the spurious correlation of learned embeddings could be eliminated. We analyze the regularizer in causal view and it motivates us to differentiate the weights of the variables based on their contribution on the confounding bias. Then, these sample weights are used for reweighting GNNs to eliminate the estimation bias, thus help to improve the stability of prediction on unknown test nodes. Comprehensive experiments are conducted on several challenging graph datasets with two kinds of label selection biases. The results well verify that our proposed model outperforms the state-of-the-art methods and DGNN is a flexible framework to enhance existing GNNs.

\end{abstract}


%
\IEEEpeerreviewmaketitle

\section{Introduction}

\IEEEPARstart{G}raph Neural Networks (GNNs) are powerful deep learning algorithms on graphs with various applications~\cite{scarselli2008graph,kipf2016semi,velivckovic2017graph,hamilton2017inductive}. Existing GNNs mainly learn a node embedding through aggregating the features from its neighbors, and such message-passing framework is supervised by the node label in an end-to-end manner. During this training procedure, GNNs will effectively learn the correlation between the structure patterns and node features with the node labels, so that GNNs are capable of learning the embeddings of new nodes and inferring their labels.

\par One basic requirement of GNNs making precise predictions on unseen test nodes is that the distribution of labeled training nodes and test nodes is the same, i.e., the structure and feature of labeled training and test nodes follow the similar pattern, so that the learned correlation between the current graph and labels can be well generalized to the new nodes. However, in reality, there are two inevitable issues. (1) Because it is difficult to control the graph collection in an unbiased manner, the relationship between the collected real-world graph and the labeled nodes is inevitably biased. Training on such a graph will cause biased correlation with node labels. Taking a scientist collaboration network as an example, if most scientists with ``machine learning'' (ML) label collaborate with those with ``computer vision'' (CV) label, existing GNNs may learn spurious correlation, i.e., scientists who cooperate with CV scientists are ML scientists. If a new ML scientist only connects with ML scientists or the scientists in other areas, it will be probably misclassified. (2) The test nodes in the real scenario are usually not available in the training phase, implying that the distribution of new nodes is agnostic. Once the distribution is inconsistent with that in the training nodes, the performance of all the current GNNs will be hindered. Even transfer learning is able to solve the distribution shift problem, however, it still needs the prior of test distribution, which actually cannot be obtained beforehand. Therefore, the agnostic label selection bias greatly affects the generalization ability of GNNs on unknown test data.

\par In order to observe selection bias in real graph data, we conduct an experimental investigation to validate the effect of selection bias on GNNs (see Section~\ref{sec::Experimental Investigation}). We select training nodes with different biased degrees for each dataset, making the distribution of training nodes and test nodes inconsistent. The results clearly show that selection bias drastically hinders the performance of GNNs on unseen test nodes. Moreover, with heavier bias, the performance drops more. Further, we theoretically analyze how the data selection bias results in the estimation bias in GNN parameters (see Section~\ref{sec::Theoretical}). Based on the stable learning technique~\cite{kuang2020stable}, we can assume that the learned embeddings consist of two parts: stable variables and unstable variables. The data selection bias will cause spurious correlation between these two kinds of variables. Thereby we prove that with the inevitable model misspecification, the spurious correlation will further cause the parameter estimation bias. Once the weakness of the current GNNs with selection bias is identified, one natural question is ``\textit{how to remove the estimation bias in GNNs?}''

\par In this paper, we propose a novel Debiased Graph Neural Network (DGNN) framework for stable graph learning by jointly optimizing a differentiated decorrelation regularizer and a weighted GNN model. Specifically, the differentiated decorrelation regularizer is able to learn a set of sample weights under differentiated variable weights, so that the spurious correlation between stable and unstable variables would be greatly eliminated. Based on the causal view analysis of the decorrelation regularizer, we theoretically prove that the weights of variables can be differentiated by the regression coefficients. Compared with existing decorrelation methods~\cite{kuang2020stable,shen2020stable}, the proposed regularizer is able to remove the spurious correlation while maintaining a higher effective sample size and requiring less prior knowledge. Moreover, to better combine the decorrelation regularizer with existing GNN architecture, the theoretical result shows that adding the regularizer to the embeddings learned by the penultimate layer could be both theoretically sound and flexible. Then the sample weights learned by the decorrelation regularizer are used to reweight the GNN loss so that the parameter estimation could be unbiased.

\par In summary, the contributions of this paper are three-fold: i) We investigate a new problem of learning GNNs with agnostic label selection bias. The problem setting is general and practical for real applications. ii) We bring the idea of variable decorrelation into GNNs to relieve bias influence on model learning and propose a general framework DGNN that could be adopted to various GNNs. iii) We conduct experiments on real-world graph benchmarks with two kinds of agnostic label selection biases, and the experimental results demonstrate the effectiveness and flexibility of our model.

\section{Effect of Label Selection Bias on GNNs}
\label{sec::effect}

In this section, we first summarize the main notations used in this paper in Table~\ref{Tab:notations} and then formulate our target problem:
\begin{myPro}[\textbf{Semi-supervised Learning on Graph with Agnostic Label Selection Bias}]
Given a training graph $\mathcal{G}_{train} =\{\mathbf{A}_{train}, \mathbf{X}_{train}, \mathbf{Y}_{train}\}$, where $\mathbf{A}_{train}\in\mathbb{R}^{N\times N}$ ($N$ nodes) represents the adjacency matrix, $\mathbf{X}_{train}\in\mathbb{R}^{N\times D}$ ($D$ features) refers to the node feature vectors and $\mathbf{Y}_{train}\in\mathbb{R}^{n\times C}$ ($n$ labeled nodes, $C$ classes) refers to the available labels for training ($n\ll N$), the task is to learn a GNN $g_\theta(\cdot)$ with parameter $\theta$  to precisely predict the label of nodes on test graph $\mathcal{G}_{test} =\{\mathbf{A}_{test}, \mathbf{X}_{test}, \mathbf{Y}_{test}\}$, where distribution $\Psi(\mathcal{G}_{train})\neq\Psi(\mathcal{G}_{test})$.
\end{myPro}

\begin{table}[ht]
\caption{Glossary of Notations.}
\begin{tabular}{r|l}
\hline
\textbf{Notation} & \textbf{Description} \\ \hline
  $\mathcal{G}_{train}$      &  Training graph         \\ 
 $\mathcal{G}_{test}$      & Test graph           \\ 
  $\mathbf{A}_{train/test}$        & The adjacency matrix of $\mathcal{G}_{train}$ or $\mathcal{G}_{test}$          \\
  $\mathbf{X}_{train/test}$        & The node feature vectors of $\mathcal{G}_{train}$ or $\mathcal{G}_{test}$          \\
   $\mathbf{Y}_{train/test}$        & The node label vectors of $\mathcal{G}_{train}$ or $\mathcal{G}_{test}$          \\
    $\mathbf{H}$/$\hat{\mathscr{G}}(\mathbf{X},\mathbf{A};\theta_g)$        & Node embeddings matrix learned by GNNs          \\
  $\mathbf{S}$        & The stable variables in  $\mathbf{H}$         \\
  $\mathbf{V}$        & The unstable variables in  $\mathbf{H}$         \\
  $\bar{\mathbf{S}}$& The latent stable variables to generate $\mathbf{Y}$         \\
  $\bar{\mathbf{V}}$        & The unstable variables to generate label $\mathbf{Y}$         \\
  $\tilde{\beta}_S$ & The linear coefficients for $\mathbf{S}$         \\
  $\tilde{\beta}_V$        & The linear coefficients for $\mathbf{V}$         \\
  $g(\cdot)$ & The non-linear transformation for stable variables $\mathbf{S}$ \\
  $\beta_S$ & The linear coefficients for $\bar{\mathbf{S}}$         \\
  $\beta_V$        & The linear coefficients for $\bar{\mathbf{V}}$         \\
   $T$        & Treatment variable         \\
   $\mathbf{X}$  & Confounders \\
   $\mathbf{w}$ & Sample weights \\
   $\alpha$ & Variable weights in DVD term \\
   $\xi$  & The linear coefficients for confounders $\mathbf{X}$ \\
   $\gamma$  & The linear coefficient for treatment $T$ \\
   $Y_i(t)$ & The potential outcome of sample $i$ with treatment $t$ \\
  \hline
\end{tabular}
\label{Tab:notations}
\end{table}
\subsection{Experimental Investigation}
\label{sec::Experimental Investigation}
\begin{figure*}[t]
\centering
\subfigure[Cora]{
\includegraphics[width=1.8in]{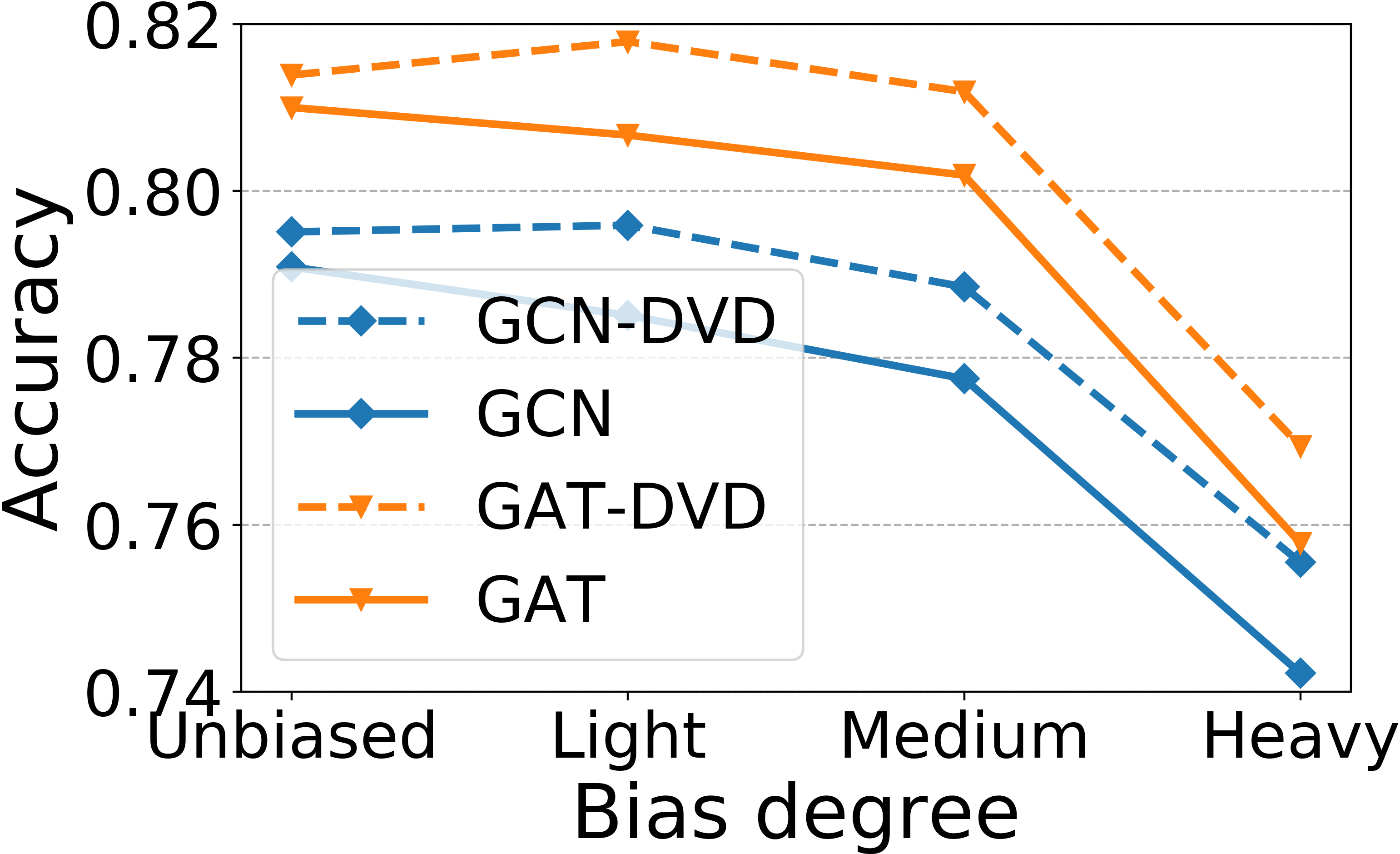}
\label{fig:gradient}
}
\hspace{1pt}
\subfigure[Citeseer]{
\includegraphics[width=1.8in]{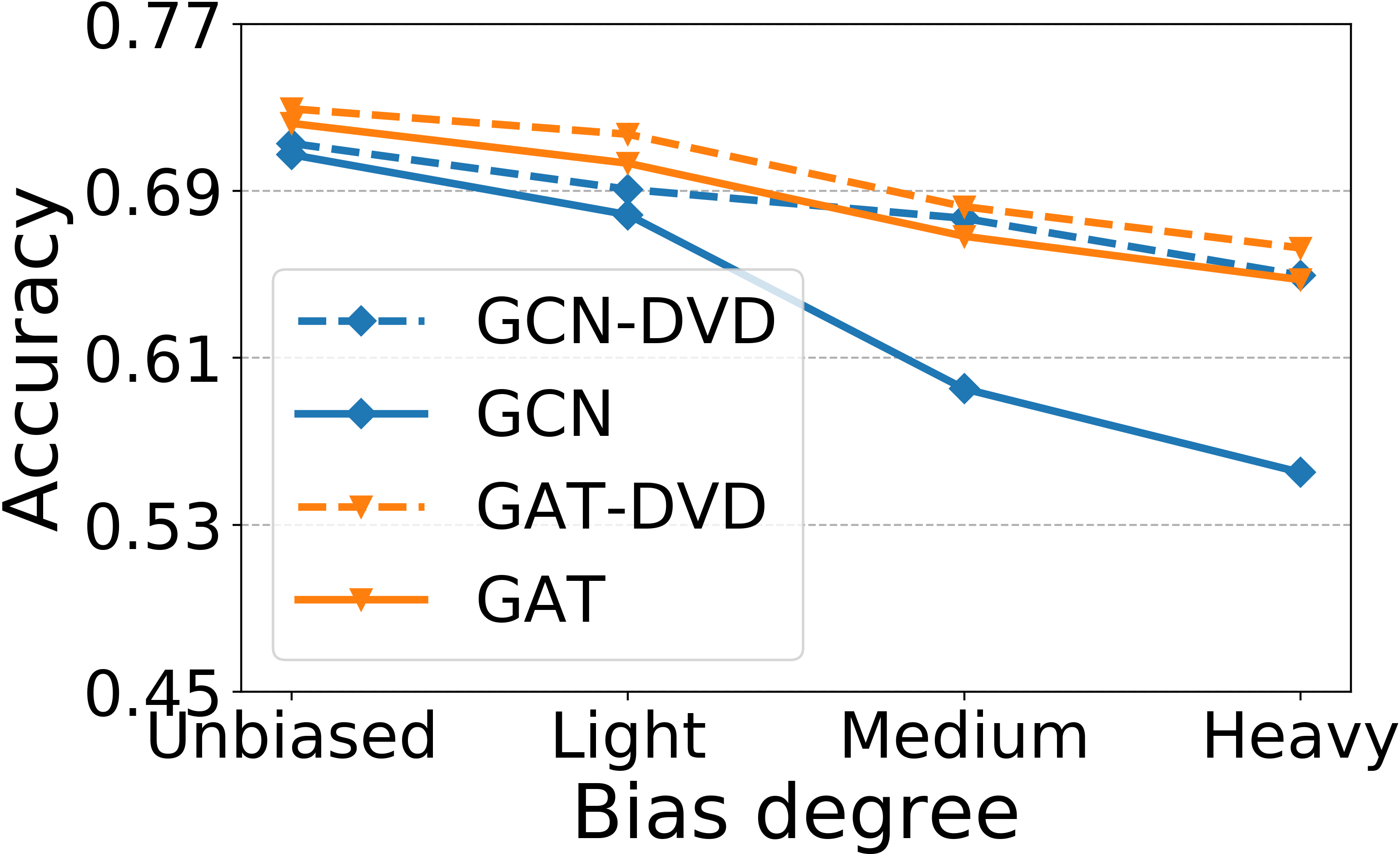}
\label{fig:gradient}
}
\hspace{1pt}
\subfigure[Pubmed]{
\includegraphics[width=1.8in]{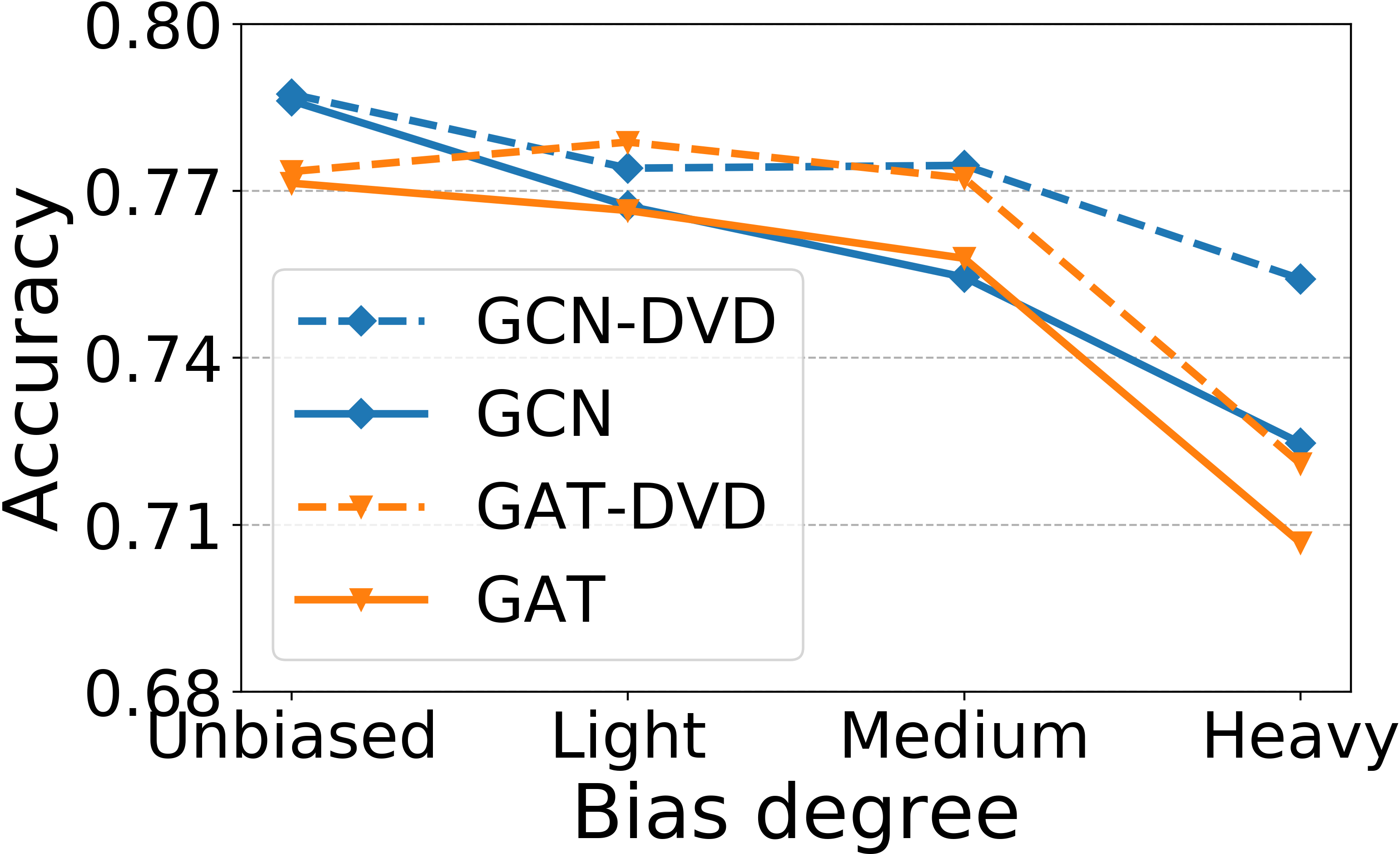}
\label{fig:123}
}
\caption{Effect of selection bias on GCN and GAT.}
\label{fig:preliminary}
\end{figure*}
We conduct an experimental investigation to examine whether the existing GNNs are sensitive to the distribution shifts caused by the label selection bias. One motivating example is that due to the research interests of the researcher that they are more likely to label the interdisciplinary documents that cite more papers from different subjects in a citation network, while testing may be conducted on nodes with any neighborhood distribution. Based on this example, the main idea of the experimental investigation is that we will perform two representative GNNs: GCN~\cite{kipf2016semi} and GAT~\cite{velivckovic2017graph} on three widely used graph datasets: \textit{Cora}, \textit{Citeseer}, \textit{Pubmed}~\cite{sen2008collective} with different bias degrees. If the performance drops sharply comparing with the scenarios without selection bias, this will demonstrate that GNNs cannot generalize well in selection bias settings.

\par To simulate the agnostic selection bias scenario, we first follow the inductive setting in~\cite{pmlr-v97-wu19e} that masks the validation and test nodes as the training graph $\mathcal{G}_{train}$ in the training phase, and then infer the labels of validation and test nodes with whole graph $\mathcal{G}_{test}$. In this way, the distribution of test node can be considered agnostic. Following~\cite{zadrozny2004learning}, we design a biased label selection method on training graph $\mathcal{G}_{train}$.  The selection variable $e$ is introduced to control whether the node will be selected as labeled nodes, where $e=1$ means selected and 0 otherwise. For node $i$, we calculate its neighbor distribution ratio: $r_i=|\{j|j\in\mathcal{N}_i, y_j\neq  y_i\}|/|\mathcal{N}_i|$, where $\mathcal{N}_i$ is neighborhood of node $i$ in $\mathcal{G}_{train}$ and $y_j\neq  y_i$ means the label of central node $i$ is not the same as the label of its neighborhood node $j$. And $r_i$ measures the difference between the label of central node $i$ with the labels of its neighborhood. Then we average all the nodes' $r$ to get a threshold $t$. For each node, the probability to be selected is: $P(e_i=1|r_i)=\left\{
\begin{array}{rcl}
\epsilon       &      {r_i\geq t}\\
1-\epsilon     &      {r_i<t}\\
\end{array} \right. $, where $\epsilon\in (0.5,1)$ is used to control the degree of selection bias and a larger $\epsilon$ means a heavier bias. We set $\epsilon$ as \{0.7, 0.8, 0.9\} to get three bias degrees for each dataset, termed as \textit{Light, Medium, Heavy}, respectively. We select 20 nodes for each class for training and the validation and test nodes are the same as~\cite{yang2016revisiting}. Furthermore, we take the \textit{unbiased} datasets as baselines, where the labeled nodes are selected randomly.

\par Figure~\ref{fig:preliminary} is the results of GCN, GAT and our proposed method, GCN/GAT-DVD, on these datasets with four bias degrees. We can find that: i) Compared with the unbiased scenario, when performing GCN/GAT on biased datasets, they suffer from serious performance decrease, indicating that selection bias greatly affects the GNNs' performance. ii)
All lines decrease monotonically with the increase of bias degree, demonstrating that heavier biases will cause larger performance reduction. iii) GCN/GAT-DVD outperforms the corresponding base models (i.e., GCN/GAT) consistently and achieves larger improvements in heavier bias degree scenarios, indicating that our proposed method could relieve the effect of selection bias.

\subsection{Theoretical Analysis}
\label{sec::Theoretical}
The above experiment empirically verifies the effect of selection bias on GNNs. Here we theoretically analyze the effect of selection bias on estimating the parameters in GNNs. First, because biased labeled nodes have biased neighborhood structure and features, GNNs will encode this biased information
into the node embeddings, which is validated by the experimental investigation. Based on stable learning technique~\cite{kuang2020stable},  we make the following assumption:
\begin{myAss}
The node embeddings learned by GNNs for each node can be decomposed as $\mathbf{H}=\{\mathbf{S},\mathbf{V}\}$, where $\mathbf{S}$ represents the stable variables and $\mathbf{V}$ represents the unstable variables. Specifically, for both training and test environments, $\mathbb{E}(\mathbf{Y}|\mathbf{S}=s,\mathbf{V}=v)=\mathbb{E}(\mathbf{Y}|\mathbf{S}=s)$.
\label{Assumption 1}
\end{myAss}

\par Under Assumption~\ref{Assumption 1}, the distribution shift between training set and test set is mainly induced by the variation in the joint distribution over ($\mathbf{S}$, $\mathbf{V}$), i.e., $\mathbb{P}(\mathbf{S}_{train}, \mathbf{V}_{train})\neq\mathbb{P}(\mathbf{S}_{test}, \mathbf{V}_{test})$. However, there is an invariant relationship between stable variables $\mathbf{S}$ and outcome $\mathbf{Y}$ in both training and test environments, which can be expressed as $\mathbb{P}(\mathbf{Y}_{train}|\mathbf{S}_{train})=\mathbb{P}(\mathbf{Y}_{test}|\mathbf{S}_{test})$. Assumption~\ref{Assumption 1} can be guaranteed by $\mathbf{Y}\bot \mathbf{V}|\mathbf{S}$. Thus, one can solve the stable prediction problem by developing a function  $g(\cdot)$ based on $\mathbf{S}$. However, one can hardly identify such variables in GNNs.
\par Without loss of generality, we take $\mathbf{Y}$ as a continuous variable for analysis and have the following assumption:
\begin{myAss}
The true generation process of target variable $\mathbf{Y}$ contains not only the linear combination of stable variables $\mathbf{S}$, but also the nonlinear transformation of stable variables.
\label{Assumption 2}
\end{myAss}
\par Based on the above assumptions, we formalize the label generation process as follows:
\begin{equation}
    \begin{aligned}
    \mathbf{Y}=f(\mathbf{X}, \mathbf{A}) + \varepsilon &= \bar{\mathbf{S}}\beta_S + \bar{\mathbf{V}}\beta_V + g(\bar{\mathbf{S}}) + \varepsilon,
    \end{aligned}
	\label{equ::data generation}
\end{equation}
where $\bar{\mathbf{S}}$ and $\bar{\mathbf{V}}$ are latent stable variables and unstable variables to generate label $\mathbf{Y}$, which can be learned by GNNs from raw graph data, $\beta_S$ and $\beta_V$ are the corresponding linear coefficients and they represent the effect of each latent variable on outcome $\mathbf{Y}$, $\varepsilon$ is the independent random noise, and $g(\cdot)$ is the nonlinear transformation function of stable variables. According to Assumption~\ref{Assumption 1}, we know that coefficients of unstable variables $\bar{\mathbf{V}}$ are actually $\mathbf{0}$ (i.e., $\beta_V=\mathbf{0}$).

\par  For a classical GNN model with a linear regression predictor, its prediction function can be formulated as:
\begin{equation}
     \hat{\mathbf{Y}} = \hat{\mathscr{G}}(\mathbf{X},\mathbf{A};\theta_g)_S\hat{\beta}_S + \hat{\mathscr{G}}(\mathbf{X},\mathbf{A};\theta_g)_V\hat{\beta}_V + \varepsilon,
    \label{equ::GNN_Y}
\end{equation}
where $\hat{\mathscr{G}}(\mathbf{X},\mathbf{A};\theta_g)\in \mathbb{R}^{N\times p}$  denotes the node embeddings learned by a GNN, such as GCN and GAT, the output variables of $\hat{\mathscr{G}}(\mathbf{X},\mathbf{A};\theta_g)$ can be decomposed as stable variables $
\hat{\mathscr{G}}(\mathbf{X},\mathbf{A};\theta_g)_S\in\mathbf{R}^{N\times m}$ and unstable variables $\hat{\mathscr{G}}(\mathbf{X},\mathbf{A};\theta_g)_V\in\mathbf{R}^{N\times q}$ ($m+q=p$) corresponding to $\bar{\mathbf{S}}$ and $\bar{\mathbf{V}}$ in Eq. (\ref{equ::data generation}).  Compared with Eq. (\ref{equ::data generation}), we can find that the parameters of GNN model could be unbiasedly estimated if the nonlinear term $g(\bar{\mathbf{S}})=0$ (i.e., there does not exist any non-linear relationship in the label generation process that cannot be learned by GNNs), because the GNN model in Eq. (\ref{equ::GNN_Y}) will have the same label generation mechanism as Eq. (\ref{equ::data generation}). However, as common used GNNs only have several layers which may limit their nonlinear power and the real-world graph data is far more complicated, it is reasonable to assume that there is a nonlinear term $g(\bar{\mathbf{S}})\neq 0$ that cannot be fitted by the GNNs. Under this assumption, next, we taking a vanilla GCN~\cite{kipf2016semi} as an example to illustrate how the distribution shift will induce parameter estimation bias. A two-layer GCN can be formulated as $\hat{\mathbf{A}}\sigma(\hat{\mathbf{A}}\mathbf{X}\mathbf{W}^{(0)})\mathbf{W}^{(1)}$, where $\hat{\mathbf{A}}$ is the normalized adjacency matrix, $\mathbf{W}$ is the transformation matrix at each layer and $\sigma(\cdot)$ is the Relu activation function. We decompose GCN as two parts: one is embedding learning part $\hat{\mathbf{A}}\sigma(\hat{\mathbf{A}}\mathbf{X}\mathbf{W}^{(0)})$, which can be decomposed as $[\mathbf{S}^\mathrm{T}, \mathbf{V}^\mathrm{T}]$, corresponding to $\hat{\mathscr{G}}(\mathbf{X},\mathbf{A};\theta_g)_S$ and $\hat{\mathscr{G}}(\mathbf{X},\mathbf{A};\theta_g)_V$ in Eq. (\ref{equ::GNN_Y}), and the other part is $\mathbf{W}^{(1)}$, where the learned parameters can be decomposed as $[\tilde{\beta}_S, \tilde{\beta}_V]$, corresponding to $[\hat{\beta}_S, \hat{\beta}_V]$ in Eq. (\ref{equ::GNN_Y}). We aim at minimizing the least-square loss:
\begin{equation}\label{equ::square loss}
  \mathcal{L}_{GCN}=\sum^n_{i=1}(\mathbf{S}_i^\mathrm{T}\tilde{\beta}_S+\mathbf{V}^\mathrm{T}_i\tilde{\beta}_V-\mathbf{Y}_i)^2.
\end{equation}

\par According to the derivation rule of partitioned regression model~\cite{kuang2020stable,nurhonen1992property}, with $\mathbf{S}=\bar{\mathbf{S}}$ and $\mathbf{V}=\bar{\mathbf{V}}$, we have:
\begin{equation}
\begin{aligned}
\tilde{\beta}_{V}-\beta_V&=(\frac{1}{n}\sum^n_{i=1}\mathbf{V}_i^\mathrm{T}\mathbf{V}_i)^{-1}(\frac{1}{n}\sum^n_{i=1}\mathbf{V}_i^\mathrm{T} g(\mathbf{S}_i)) \\
&+ (\frac{1}{n}\sum^n_{i=1}\mathbf{V}_i^\mathrm{T}\mathbf{V}_i)^{-1}(\frac{1}{n}\sum^n_{i=1}\mathbf{V}_i^\mathrm{T} \mathbf{S}_i) (\beta_S-\tilde{\beta}_{S}),
  \label{equ:V}
\end{aligned}
\end{equation}

\begin{equation}
\begin{aligned}
 \tilde{\beta}_{S}-\beta_S&=(\frac{1}{n}\sum^n_{i=1}\mathbf{S}_i^\mathrm{T}\mathbf{S}_i)^{-1}(\frac{1}{n}\sum^n_{i=1}\mathbf{S}_i^\mathrm{T} g(\mathbf{S}_i)) \\
 &+ (\frac{1}{n}\sum^n_{i=1}\mathbf{S}_i^\mathrm{T}\mathbf{S}_i)^{-1}(\frac{1}{n}\sum^n_{i=1}\mathbf{S}_i^\mathrm{T} \mathbf{V}_i) (\beta_V-\tilde{\beta}_{V}),
 \label{equ:S}
\end{aligned}
\end{equation}
where $n$ is labeled node size, $\mathbf{S}_i$ is $i$-th sample of $\mathbf{S}$, $\frac{1}{n}\sum^n_{i=1}\mathbf{V}_i^\mathrm{T} g(\mathbf{S}_i)=\mathbb{E}(\mathbf{V}^\mathrm{T} g(\mathbf{S}))+o_p(1)$, $\frac{1}{n}\sum^n_{i=1}\mathbf{V}_i^\mathrm{T} \mathbf{S}_i=\mathbb{E}(\mathbf{V}^\mathrm{T}\mathbf{S})+o_p(1)$ and $o_p(1)$ is the error which is negligible. Ideally, $\tilde{\beta}_{V}-\beta_V=0$ indicates that there is no bias between the estimated and the real parameter. However, if $\mathbb{E}(\mathbf{V}^\mathrm{T} \mathbf{S})\neq 0$ or $\mathbb{E}(\mathbf{V}^\mathrm{T} g(\mathbf{S}))\neq 0$ in Eq.~(\ref{equ:V}), $\tilde{\beta}_{V}$ will be biased, leading to the biased estimation on $\tilde{\beta}_{S}$ in Eq.~(\ref{equ:S}) as well, i.e, the true effect of learned embeddings on label $\mathbf{Y}$ can not be estimated precisely. Since the correlation between $\mathbf{V}$ and $\mathbf{S}$ (or $g(\mathbf{S})$)~\footnote{We assume all variables are centered with zero mean. This assumption could be satisfied by adding a normalization layer after the learned embeddings.} might shift in test phase, the biased parameters learned in training set is not the optimal parameters for predicting testing nodes. Therefore, to increase the stability of prediction, we need to unbiasedly estimate the parameters of $\tilde{\beta}_V$ by removing the correlation between $\mathbf{V}$ and $\mathbf{S}$ (or $g(\mathbf{S})$) on training graph, making $\mathbb{E}(\mathbf{V}^\mathrm{T} \mathbf{S})= 0$ and $\mathbb{E}(\mathbf{V}^\mathrm{T} g(\mathbf{S}))= 0$. Note that $\frac{1}{n}\sum^n_{i=1}\mathbf{S}_i^\mathrm{T} g(\mathbf{S}_i)$ in Eq.~(\ref{equ:S}) can also cause estimation bias, but the relation between $\mathbf{S}$ and $g(\mathbf{S})$ is stable across environments, which do not affect the stability to some extent.

\begin{figure*}[t]
  \centering
  \includegraphics[width=17cm]{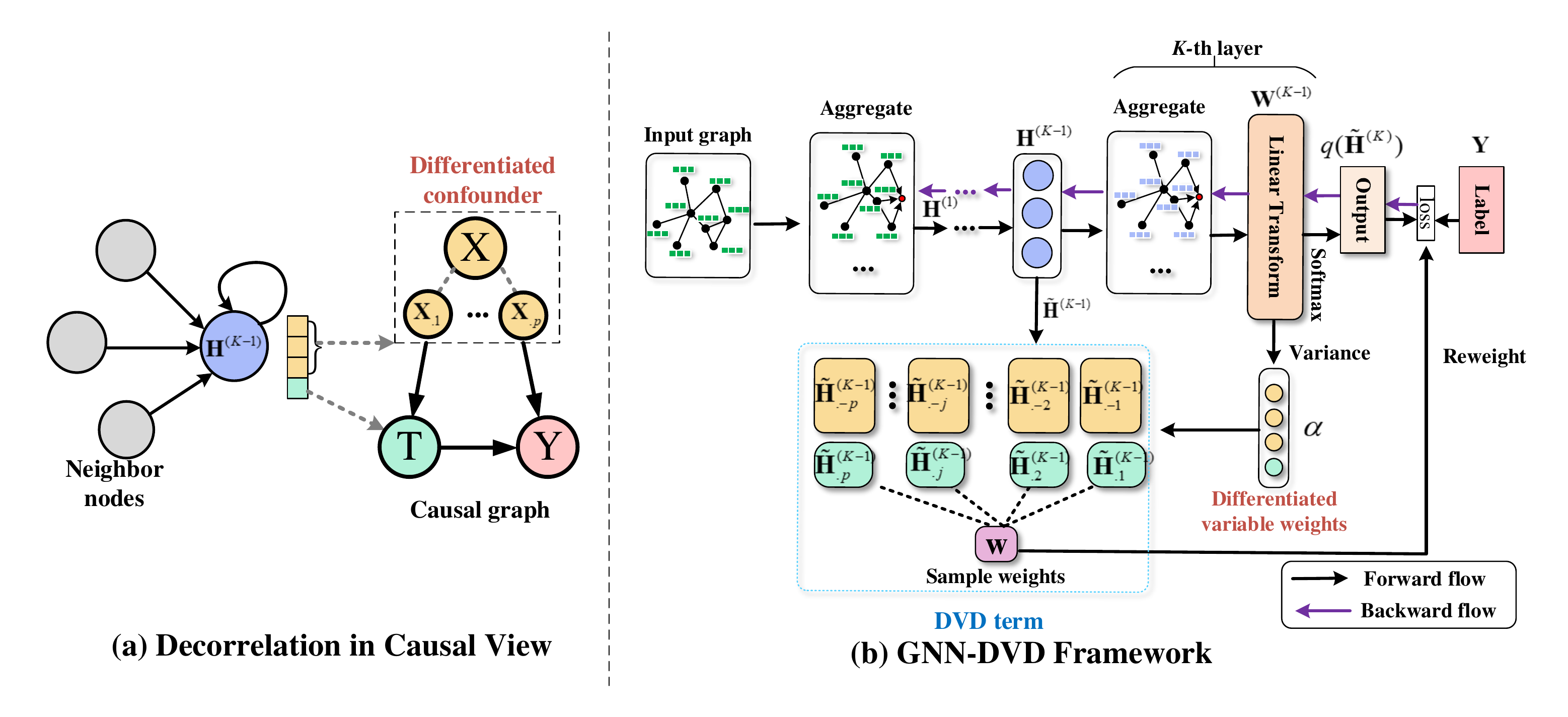}
  \caption{(a) Diagram of decorrelating node embeddings with confounding balance. $\tilde{\mathbf{H}}^{(K-1)}$ is the node embedding matrix to be decorrelated. $T$ is the treatment variable, corresponding to one target variable in $\tilde{\mathbf{H}}^{(K-1)}$. $\mathbf{X}$ means the confounders, corresponding to the remaining variables of the target variable in $\tilde{\mathbf{H}}^{(K-1)}$. $Y$ is the outcome, corresponding to labels.  (b) The framework of GNN-DVD. The same color in the two figures represents the same kind of variable. }
  \label{fig::Model}
\end{figure*}
\section{Proposed Model}

\subsection{Revisiting on Variable Decorrelation in Causal View}

\label{sec::confounder balanceing}
\par To decorrelate $\mathbf{V}$ and $\mathbf{S}$ (or $g(\mathbf{S})$), we should decorrelate the output variables of $\hat{\mathscr{G}}(\mathbf{X},\mathbf{A};\theta_g)$. \cite{kuang2020stable} proposes a Variable Decorrelation (VD) term with sample reweighting technique to eliminate the correlation between each variable pair, in which the sample weights are learned by jointly minimizing the moment discrepancy between each variable pair:
\begin{equation}
   \mathcal{L}_{VD}(\mathbf{H})=\sum^p_{j=1}||\mathbf{H}_{.j}^\mathrm{T}\Lambda_\mathbf{w}\mathbf{H}_{.-j}/n-\mathbf{H}_{.j}^\mathrm{T} \mathbf{w}/n\cdot\mathbf{H}_{.-j}^\mathrm{T} \mathbf{w}/n||^2_2,
	\label{equ:L_VD}
\end{equation}
where $\mathbf{H}\in\mathbb{R}^{n\times p}$ means the variables needed to be decorrelated, i.e., $\hat{\mathscr{G}}(\mathbf{X},\mathbf{A};\theta_g)$ of GNNs, $\mathbf{H}_{.j}$ is $j$-th variable of $\mathbf{H}$, $\mathbf{H}_{.-j}=\mathbf{H}\backslash{\mathbf{H}_{.j}}$ means all the remaining variables by setting the value of $j$-th variable in $\mathbf{H}$ as zero, $\mathbf{w}\in\mathbb{R}^{n\times 1}$ is the sample weights, $\sum_{i=1}^n \mathbf{w}_i=n$ and $\Lambda_\mathbf{w}=\text{diag}(\mathbf{w}_1,\cdots,\mathbf{w}_n)$ is the corresponding diagonal matrix. As we can see, $\mathcal{L}_{VD}(\mathbf{H})$ can be reformulated as $\sum_{j\neq k}||\mathbf{H}_{.j}^\mathrm{T}\Lambda_\mathbf{w}\mathbf{H}_{.k}/n-\mathbf{H}_{.j}^\mathrm{T} \mathbf{w}/n\cdot\mathbf{H}_{.k}^\mathrm{T} \mathbf{w}/n||^2_2$, and it aims to let $\mathbb{E}(\mathbf{H}_{.j}^\mathrm{T} \mathbf{H}_{.k})= \mathbb{E}(\mathbf{H}_{.j}^\mathrm{T})\mathbb{E}(\mathbf{H}_{.k})$ for each variable pair $j$ and $k$. $\mathcal{L}_{VD}(\mathbf{H})$ decorrelates all the variable pairs equally. However, decorrelating all the variables requires sufficient samples~\cite{kuang2020stable}, i.e., $n\rightarrow\infty$, which is hard to be satisfied, especially in the semi-supervised setting. In this scenario, we cannot guarantee $\mathcal{L}_{VD}(\mathbf{H})=0$.  Therefore, the key challenge is the difficulties of removing the spurious correlation that has the largest impact on the unbiased estimation when $\mathcal{L}_{VD}(\mathbf{H})\neq0$.

\par Inspired by confounding balancing technique in observational studies~\cite{hainmueller2012entropy}, we revisit the VD regularizer in causal view and show how to differentiate each variable pair. Confounding balancing techniques are often used for causal effect estimation of treatment $T$, where the distributions of
confounders $\mathbf{X}$ are different between treated ($T=1$) and control ($T=0$) groups because
of non-random treatment assignment. One could balance the distribution of confounders between treatment and control groups to unbiasedly estimate causal treatment effects~\cite{kuang2020causal,kuang2018stable,kuang2020data}.  Most balancing approaches exploit moments to characterize
distributions, and balance them by adjusting sample weights $\mathbf{w}$ as follows: $\mathbf{w} = \arg\min\limits_{\mathbf{w}}||\sum_{i:T_i=1}\mathbf{X}_{i}-\sum_{i:T_i=0}\mathbf{w}_i\cdot \mathbf{X}_{i}||^2_2.$ After balancing, the treatment $T$ and the confounders $\mathbf{X}$ tend to be independent.
\par Given one targeted variable $j$, its decorrelation term, $\mathcal{L}_{VD_j}=||\mathbf{H}_{.j}^\mathrm{T}\Lambda_\mathbf{w}\mathbf{H}_{.-j}/n-\mathbf{H}_{.j}^\mathrm{T} \mathbf{w}/n\cdot\mathbf{H}_{.-j}^\mathrm{T} \mathbf{w}/n||^2_2$, is to make $\mathbf{H}_{.j}$ independent from $\mathbf{H}_{.-j}$ \footnote{Nonlinear relationships between variables can be incorporated by considering high-order moments in Eq. (\ref{equ:L_VD}), for example, a polynomial augmented function $f(\mathbf{H})=(\mathbf{H}, \mathbf{H}^2, \mathbf{H}_{.i}\mathbf{H}_{.j}, \mathbf{H}^3, \mathbf{H}_{.i}\mathbf{H}_{.j}\mathbf{H}_{.k}, \cdots)$.}, which is same as the confounding balancing term making treatment and confounders independent. Thereby, $\mathcal{L}_{VD_j}$ can also be viewed as a confounding balancing term, where $\mathbf{H}_{.j}$ is treatment and $\mathbf{H}_{.-j}$ is confounders, illustrated in Fig.~\ref{fig::Model}(a). Hence, our target can be explained as unbiased estimation of causal effect of each variable which is invariant across training and test set. As different variable may contribute unequally to the confounding bias, it is necessary to differentiate the confounders. The target of differentiating confounders exactly matches our target that removing the correlation of variables that has the largest impact on the unbiased estimation.

\subsection{Differetiated Variable Decorrelation}
\par Considering a continuous treatment, the causal effect of treatment can be measured by Marginal Treatment Effect Function (MTEF)~\cite{kreif2015evaluation}, and defined as: $MTEF=\frac{\mathbb{E}[Y_i(t)]-\mathbb{E}[Y_i(t-\Delta t)]}{\Delta t}$, where $Y_i(t)$ represents the potential outcome of sample $i$ with treatment status $T=t$, and $\Delta t$ denotes the increasing level of treatment. With the sample weights $\mathbf{w}$ decorrelating the treatment and the confounders, we can estimate the MTEF by:
\begin{equation}\label{hat_MTEF}
  \widehat{MTEF}=\frac{\frac{1}{n_i}\sum_{i:T_i=t}\mathbf{w}_i\cdot Y_i(T_i)-
  \frac{1}{n_j}\sum_{j:T_j=t-\Delta t}\mathbf{w}_j\cdot Y_j(T_j)}{\Delta t},
\end{equation}
where $n_i$ and $n_j$ are the number of samples for two groups, respectively.
Next, we theoretically analyze how to differentiate confounders' weights with the following theorem.
\begin{myTheo}
In observational studies, different confounders make unequal confounding bias on Marginal Treatment Effect Function (MTEF) with their own weights, and the weights can be learned via regressing outcome $Y$ on confounders $\mathbf{X}$ and treatment variable $T$.
\label{Theorem1}
\end{myTheo}
\begin{proof}
Recalling the Assumption~\ref{Assumption 1}, we rewrite the label generation process Eq. (\ref{equ::data generation}) under MTEF setting as:
\begin{equation}\label{equ::METF_Y}
  Y(t)=\sum_k \xi_{k}\mathbf{X}_{.k}+\gamma t+g_{t}(\mathbf{S}) + \varepsilon,
\end{equation}
where $\alpha=[\xi, \gamma]$ are the linear coefficients, and $g_{T=t}(\mathbf{S})$ is the output of nonlinear transformation of the stable variables $\mathbf{S}$ when treatment $T$ is $t$. Note that if $T\notin \mathbf{S}$, changing the value of $T$ will not change the value of $g_{t}(\mathbf{S})$, $T\in \mathbf{S}$ otherwise.

Under above formulation, we write estimator of $\widehat{MTEF}$ as:

\begin{equation}
    \begin{aligned}
    &\widehat{MTEF}=\frac{\frac{1}{n_i}\sum_{i:T_i=t}\mathbf{w}_i Y_i(T_i)-\frac{1}{n_j}\sum_{j:T_j=t-\Delta t}\mathbf{w}_j Y_j(T_j)}{\Delta t}\\
    &=\frac{\frac{1}{n_i}\sum_{i:T_i=t}\mathbf{w}_i (\sum_k \xi_{k}\mathbf{X}_{ik}+\gamma t+g_{T=t}(\mathbf{S}_i)+\epsilon)}{\Delta t}\\
    &-\frac{\frac{1}{n_j}\sum_{j:T_j=t-\Delta t}\mathbf{w}_j (\sum_{k}\xi_{k}\mathbf{X}_{jk}+\gamma (t-\Delta t)+g_{t-\Delta t}(\mathbf{S}_i)+\epsilon)}{\Delta t} \\
    &=\frac{\frac{1}{n_i}\sum_{i:T_i=t}\mathbf{w}_i \gamma t-\frac{1}{n_j}\sum_{j:T_j=t-\Delta t}\mathbf{w}_j \gamma (t-\Delta t)}{\Delta t} \\
    &+ \frac{\frac{1}{n_i}\sum_{i:T_i=t}\mathbf{w}_i\sum_k \xi_{k}\mathbf{X}_{ik}-\frac{1}{n_j}\sum_{j:T_j=t-\Delta t}\mathbf{w}_j \sum_k\xi_{k}\mathbf{X}_{ik}}{\Delta t}\\
    & +\frac{\frac{1}{n_i}\sum_{i:T_i=t} \mathbf{w}_i g_{t}(\mathbf{S}_i)-\frac{1}{n_j}\sum_{j:T_j=t-\Delta t}\mathbf{w}_j g_{t-\Delta t}(\mathbf{S}_i)}{\Delta t} + \phi(\epsilon)\\
    &=MTEF\\
    &+\sum_{k\neq t}\xi_k(\frac{\sum_{i:T_i=t}\frac{1}{n_i}\mathbf{w}_i\mathbf{X}_{ik}-\sum_{j:T_j=t-\Delta t}\frac{1}{n_j}\mathbf{w}_j\mathbf{X}_{jk}}{\Delta t})\\
    &+\frac{\frac{1}{n_i}\sum_{i:T_i=t}\mathbf{w}_i g_{t}(\mathbf{S}_i)-\frac{1}{n_j}\sum_{j:T_j=t-\Delta t}\mathbf{w}_j g_{t-\Delta t}(\mathbf{S}_i)}{\Delta t} + \phi(\epsilon),
    \end{aligned}
	\label{equ:differentiate_MTEF}
\end{equation}
where $\frac{\frac{1}{n_i}\sum_{i:T_i=t}\mathbf{w}_i \gamma t-\frac{1}{n_j}\sum_{j:T_j=t-\Delta t}\mathbf{w}_j \gamma (t-\Delta t)}{\Delta t}$ is the ground truth of $MTEF$, $\phi(\varepsilon)$ means the noise term, and $\phi(\varepsilon)\simeq 0$ with Gaussian noise. According to the last equation, to reduce the bias of $\widehat{MTEF}$, we need regulate the term $\sum\nolimits_k\xi_k(\frac{\frac{1}{n_i}\sum_{i:T_i=t}\mathbf{w}_i\mathbf{X}_{ik}-\frac{1}{n_j}\sum_{j:T_j=t-\Delta t}\mathbf{w}_j\mathbf{X}_{jk}}{\Delta t})$ and $\frac{\frac{1}{n_i}\sum_{i:T_i=t}\mathbf{w}_i g_{t}(\mathbf{S}_i)-\frac{1}{n_j}\sum_{j:T_j=t-\Delta t}\mathbf{w}_j g_{t-\Delta t}(\mathbf{S}_i)}{\Delta t}$, where the second term has the unknown term $g_{T}(\mathbf{S}_i)$ so that we can only try to reduce the first term. $\frac{\frac{1}{n_i}\sum_{i:T_i=t}\mathbf{w}_i\mathbf{X}_{ik}-\frac{1}{n_j}\sum_{j:T_j=t-\Delta t}\mathbf{w}_j\mathbf{X}_{jk}}{\Delta t}$ means the difference of the $k$-th confounder between treated and control samples. The parameter $\xi_k$ represents the confounding bias weight of the $k$-th confounder, and it is the coefficient of $\mathbf{X}_{.k}$. Moreover, because our target is to learn the weight of each variable pair, i.e., between treatment and each confounder, we also need to learn the weight $\gamma$ of treatment $T$. Hence, according to Eq. (\ref{equ::METF_Y}), the confounder weights and treatment weight can be learned by regressing observed outcome $Y$ on confounders $\mathbf{X}$ and treatment $T$. 
\end{proof} 

\par Due to the connection between treatment effect estimation with variable decorrelation as analyzed in Section~\ref{sec::confounder balanceing}, we utilize Theorem~\ref{Theorem1} to reweight the variable weight in variable decorrelation term. When apply the Theorem~\ref{Theorem1} to GNNs, the confounders $\mathbf{X}$ should be $\mathbf{H}_{.-j}$ and treatment $T$ is $\mathbf{H}_{.j}$, where the embedding $\mathbf{H}$ is learned by $\hat{\mathscr{G}}(\mathbf{X},\mathbf{A};\theta_g)$ in Eq. (\ref{equ::GNN_Y}). And the variable weights $\alpha$ is equal to the regression coefficients for $\mathbf{H}$, i.e, $\hat{\beta}$ in Eq. (\ref{equ::GNN_Y}). Then the Differentiated Variable Decorrelation (DVD) term can be formulated as follows:
\begin{equation}
    \begin{aligned}
   \min_\mathbf{w}\mathcal{L}_{DVD}(\mathbf{H})=&\sum\nolimits_{j=1}^p(\alpha^\mathrm{T}\cdot \text{abs}(\mathbf{H}_{.j}^\mathrm{T}\Lambda_\mathbf{w}\mathbf{H}_{.-j}/n\\
   &-\mathbf{H}_{.j}^\mathrm{T} \mathbf{w}/n\cdot\mathbf{H}_{.-j}^\mathrm{T} \mathbf{w}/n))^2\\
   &+ \frac{\lambda_1}{n}\sum\nolimits_{i=1}^{n}\mathbf{w}_i^2 + \lambda_2 (\frac{1}{n}\sum\nolimits_{i=1}^n\mathbf{w}_i-1)^2, \\
   &s.t. \mathbf{w}\succeq0
   \end{aligned}
	\label{equ:EXY_all_rewrite_alpha}
\end{equation}
where $\text{abs}(\cdot)$ means the element-wise absolute value operation, preventing positive and negative values from eliminating. Term $\frac{\lambda_1}{n}\sum_{i=1}^{n}\mathbf{w}_i^2$ is added to reduce the variance of sample weights to achieve stability, and the formula $\lambda_2(\frac{1}{n}\sum_{i=1}^n\mathbf{w}_i-1)^2$ avoids all the sample weights to be 0. The term $\mathbf{w}\succeq0$ constrains each sample weight to be non-negative. After variable reweighting, the weighted decorrelation term in Eq. (\ref{equ:EXY_all_rewrite_alpha}) can be rewritten as $\sum_{j\neq k}\alpha_j^2\alpha_k^2||\mathbf{H}_{.j}^\mathrm{T}\Lambda_\mathbf{w}\mathbf{H}_{.k}/n-\mathbf{H}_{.j}^\mathrm{T} \mathbf{w}/n\cdot\mathbf{H}_{.k}^\mathrm{T} \mathbf{w}/n||^2_2$, and the weight for variable pair $j$ and $k$ would be $\alpha_j^2\alpha_k^2 $, hence it considers both the weights of treatment and confounder. Then we derive the uniqueness property of $\mathbf{w}$ as follows:
\begin{myTheo}[Uniqueness]
  If $\lambda_1 n\gg p^2+\lambda_2$, $p^2\gg \max(\lambda_1, \lambda_2)$, $|\mathbf{H}_{i,j}|\leq c$ and $|\alpha_i| \leq c$ for some constant $c$, the solution $\hat{\mathbf{w}}\in\{\mathbf{w}:|\mathbf{w}_i|\leq c\}$ to minimize Eq.~(\ref{equ:EXY_all_rewrite_alpha}) is unique.
\end{myTheo}
\begin{proof}See Appendix~\ref{appendix::Theorem 2}.\end{proof}

\subsection{Debiased GNN Framework}
\label{sec::3.3}

In this section, we describe the framework of Debiased GNN (DGNN) that incorporates VD/DVD term with GNNs in a seamless way. As analyzed in Section~\ref{sec::Theoretical}, decorrelating $\hat{\mathbf{A}}\sigma(\hat{\mathbf{A}}\mathbf{X}\mathbf{W}^{(0)})$ could make parameter estimation of GCN unbiased. However, most GNNs follow a layer-by-layer stacking architecture, and the output embedding of each layer is more easy to obtain in implementing. Since $\hat{\mathbf{A}}\sigma(\hat{\mathbf{A}}\mathbf{X}\mathbf{W}^{(0)})$ is the aggregation of the first layer embedding $\sigma(\hat{\mathbf{A}}\mathbf{X}\mathbf{W}^{(0)})$, decorrelating $\hat{\mathbf{A}}\sigma(\hat{\mathbf{A}}\mathbf{X}\mathbf{W}^{(0)})$ may lack the flexibility that incorporates VD/DVD term with other GNN architectures. Fortunately, we have the following theorem to identify a more flexible way to combine variable decorrelation with GNNs.
\begin{myTheo}
Given $p$ pairwise uncorrelated variables $\mathbf{Z}=(\mathbf{Z}_1, \mathbf{Z}_2, \cdots, \mathbf{Z}_p)$, with a linear aggregation operator $\hat{\mathbf{A}}$, the variables of $\mathbf{Y}=\hat{\mathbf{A}}\mathbf{Z}$ are still pairwise uncorrelated.
\end{myTheo}
\begin{proof} Let $\mathbf{Z}=\{\mathbf{Z}_1, \mathbf{Z}_2, \cdots, \mathbf{Z}_p\}$ be $p$ pairwise uncorrelated variables. $\forall \mathbf{Z}_{i}, \mathbf{Z}_{j}\in \mathbf{Z}$, $(\mathbf{Z}_{i}^{(1)}, \mathbf{Z}_{i}^{(2)}, \cdots, \mathbf{Z}_{i}^{(n)})$ and $(\mathbf{Z}_{j}^{(1)}, \mathbf{Z}_{j}^{(2)}, \cdots, \mathbf{Z}_{j}^{(n)})$ are $n$ \textit{simple random samples} drawn from $\mathbf{Z}_{i}$ and $\mathbf{Z}_{j}$ respectively, and have same distribution with $\mathbf{Z}_{i}$ and $\mathbf{Z}_{j}$. Given a linear aggregation matrix $\hat{\mathbf{A}}=(a_{ij})$, $\forall s, v\in (1, 2, \cdots, n)$, let $\mathbf{Y}_{i}^{(s)}=\sum_{k=1}^{n} a_{sk}\mathbf{Z}_{i}^{(k)}$ and $\mathbf{Y}_{j}^{(v)}=\sum_{l=1}^{n} a_{vl}\mathbf{Z}_{j}^{(l)}$, and we have following derivation:
\begin{equation}
\begin{aligned}
\text{Cov}(\mathbf{Y}_{i}^{(s)}, \mathbf{Y}_{j}^{(v)})&=\text{Cov}(\sum_{k=1}^{n} a_{sk}\mathbf{Z}_{i}^{(k)},\sum_{l=1}^{n} a_{vl}\mathbf{Z}_{j}^{(l)})\\
&=\sum_{k=1}^{n}\sum_{l=1}^{n}a_{sk}a_{vl}\text{Cov}(\mathbf{Z}_{i}^{(k)},\mathbf{Z}_{j}^{(l)})\\
&=\sum_{k=1}^{n}\sum_{l=1}^{n}a_{sk}a_{vl}\delta_{ij},
\end{aligned}
\nonumber
\end{equation}
where $\delta_{ij}=0$ when $i\neq j$, otherwise $\delta_{ij}=1$. Therefore, when $i\neq j$, we have $\text{Cov}(\mathbf{Y}_{i}^{(s)}, \mathbf{Y}_{j}^{(v)})=0$ and $\text{Cov}(\mathbf{Y}_{i}, \mathbf{Y}_{j})=0$. Extended the conclusion to multiple variable, $\mathbf{Y}=(\mathbf{Y}_1, \mathbf{Y}_2, \cdots, \mathbf{Y}_n)$ are pairwise uncorrelated.
\end{proof}
The theorem indicates that if the variables of embeddings $\mathbf{Z}$ are uncorrelated, after any form of linear neighborhood aggregation $\hat{\mathbf{A}}$, e.g., average, attention or sum, the variables of transformed embeddings $\mathbf{Y}$ would be also uncorrelated. Therefore, decorrelating $\sigma(\hat{\mathbf{A}}\mathbf{X}\mathbf{W}^{(0)})$ can also reduce the estimation bias. For a $K$ layers of GNN, we can directly decorrelate the output of $(K-1)$-th layer, i.e.,  $\sigma(\hat{\mathbf{A}}\cdots\sigma(\hat{\mathbf{A}}\mathbf{X}\mathbf{W}^{(0)})\cdots\mathbf{W}^{(K-2)})$ for a $K$ layers of GCN.

\par The previous analysis finds a flexible way to incorporate VD/DVD term with GNNs, however, recall that we analyze GNNs based on the least-squares loss in Eq. (\ref{equ::square loss}), and most existing GNNs are designed for the classification task. Therefore, in the following, we analyze that the previous conclusions are still applicable in classification. We consider the cases that the softmax layer is used as the output layer of GNNs and loss is the cross-entropy error function. We use the Newton-Raphson update rule~\cite{Bishop2006pattern} to bridge the gap between linear regression and multi-classification. According to the Newton-Raphson update rule, the update formula for transformation matrix $\mathbf{W}^{(K-1)}$ of the last layer of GCN can be derived:
\begin{equation}
    \begin{aligned}
    \mathbf{W}_{.j}^{(\text{new})} &= \mathbf{W}_{.j}^{(\text{old})} - (\mathbf{H}^\mathrm{T}\mathbf{R}\mathbf{H})^{-1}\mathbf{H}^\mathrm{T}(\mathbf{H}\mathbf{W}_{.j}^{(\text{old})}-\mathbf{Y}_{.j})\\
     &=(\mathbf{H}^\mathrm{T}\mathbf{R}\mathbf{H})^{-1}\{\mathbf{H}^\mathrm{T}\mathbf{R}\mathbf{H}\mathbf{W}_{.j}^{(\text{old})}-
     \mathbf{H}^\mathrm{T}(\mathbf{H}\mathbf{W}_{.j}^{(\text{old})}-\mathbf{Y}_{.j})\}\\
     &=(\mathbf{H}^\mathrm{T}\mathbf{R}\mathbf{H})^{-1}\mathbf{H}^\mathrm{T}\mathbf{R}\mathbf{z},
   \end{aligned}
	\label{equ:W_j_update}
\end{equation}
where $\mathbf{R}_{kj}=-\sum_{n=1}^{N}\mathbf{H}_{n}\mathbf{W}_{.k}^{(\text{old})}(\mathbf{I}_{kj}-\mathbf{H}_{n}\mathbf{W}_{.j}^{(\text{old})})$ is a weighing matrix and $\mathbf{I}_{kj}$ is the element of the identity matrix, and $\mathbf{z}=\mathbf{H}\mathbf{W}_{.j}^{(\text{old})}-\mathbf{R}^{-1}(\mathbf{Y}_{.j}-\mathbf{W}_{.j}\mathbf{H})$ is an \textit{effective target value}. Eq. (\ref{equ:W_j_update}) takes the form of a set of \textit{normal equations} for a weighted least-squares problem. As the weighing matrix $\mathbf{R}$ is not constant but depends on the parameter vector $\mathbf{W}^{(\text{old})}_{.j}$, we must apply the normal equations iteratively. Each iteration uses the last iteration weight vector $\mathbf{W}^{(\text{old})}_{.j}$ to calculate a revised weighing matrix $\mathbf{R}$ and regresses the target value $\mathbf{z}$ with $\mathbf{H}\mathbf{W}_{.j}^{(\text{new})}$. Therefore, the variable decorrelation can also be applied to the GNNs with softmax classifier to reduce the estimation bias in each iteration. Note that according to update formula Eq. (\ref{equ:W_j_update}), we should calculate the inverse matrix $(\mathbf{H}^\mathrm{T}\mathbf{R}\mathbf{H})^{-1}$ in each iteration, which requires high computation. In practice, we use gradient descent methods to approximate Newton-Raphson update rule and it works well in experiments.

\par Fig.~\ref{fig::Model}(b) is the framework of GNN-DVD, where we input the labeled nodes' embeddings $\tilde{\mathbf{H}}^{(K-1)}$ into the regularizer $\mathcal{L}_{DVD}(\tilde{\mathbf{H}}^{(K-1)})$. As GCN has the formula $\text{softmax}(\hat{\mathbf{A}}\mathbf{H}^{(K-1)}\mathbf{W}^{(K-1)})$, the variable weights of $\tilde{\mathbf{H}}^{(K-1)}$ used for differentiating $\mathcal{L}_{DVD}(\tilde{\mathbf{H}}^{(K-1)})$ can be computed from
\begin{equation}
    \alpha=\text{Var}(\mathbf{W}^{(K-1)}, \text{axis}=1)
    \label{Eq:alpha}
\end{equation}
where $\text{Var}(\cdot, \text{axis}=1)$ refers to calculating the variance of each row of some matrix and it reflects each variable's weight for classification which is similar to the regression coefficients. Note that when incorporating VD term with GNNs, we do not need compute the variable weights. Then the sample weights $\mathbf{w}$ learned by DVD term have the ability to remove the correlation in $\tilde{\mathbf{H}}^{(K-1)}$. Inspired by sample weighting methods~\cite{zou2020counterfactual}, we propose to use this sample weights to reweight softmax loss:
 \begin{equation}
    \min\limits_{\theta} \mathcal{L}_G = \sum\limits_{l\in\mathcal{Y}_L} \mathbf{w}_l \cdot \ln(q(\tilde{\mathbf{H}}^{(K)}_{l})\cdot \mathbf{Y}_{l}),\\
    \label{Eq:GNN_DVD}
\end{equation}
where $q(\cdot)$ is the softmax function, $\mathcal{Y}_L$ is the set of labeled node indices and $\theta$ is the set of parameters of GCN. After reweighting nodes by these weights, we can create a pseudo-population where the biases in node neighborhood are effectively reduced, with which the off-the-shelf GNN models can achieve more accurate prediction under agnostic environments. The whole algorithm is summarized in Algorithm~\ref{alg:GNN-DVD}.

\begin{algorithm}[!htbp]
    \caption{GNN-DVD Algorithm}
\label{alg:GNN-DVD}
\SetKwInOut{Input}{\textbf{Input}}\SetKwInOut{Output}{\textbf{Output}}\SetKwInOut{Initialization}{\textbf{Initialization}} 
\Input{Training graph $\mathcal{G}_{train}=\{\mathbf{A},\mathbf{X},\mathbf{Y}\}$, and indices of labeled nodes $\mathcal{Y}_L$; Max iteration:$maxIter$
		  }
		    \Output{GNN parameter $\theta$ and sample weights $\mathbf{w}$
		      }
		    \Initialization{Let $\mathbf{w}= \omega\odot\omega$ and initialize sample weights $\omega$ with $\mathbf{1}$; Initialize GNN's parameters $\theta$ with random uniform distribution; Iteration $t\leftarrow 0$}
\BlankLine 
    \While{not converged \rm{or} $t<maxIter$}
    {

      Optimize $\theta^{(t)}$ to minimize $\mathcal{L}_G$ via Eq. (\ref{Eq:GNN_DVD});\\
      Calculate variable weights $\alpha^{(t)}$ from $\mathbf{W}^{(K-1)}$ via Eq. (\ref{Eq:alpha});\\
      Optimize $\omega^{(t)}$ to minimize $\mathcal{L}_{DVD}(\tilde{\mathbf{H}}^{(K-1)})$ via Eq. (\ref{equ:EXY_all_rewrite_alpha});\\
      $t = t+1$;
    }
    \textbf{Return:} $\theta$ and $\mathbf{w}=\omega\odot\omega$
\end{algorithm}

To optimize our GNN-DVD algorithm, we propose an iterative method. Firstly, we let $\mathbf{w}= \omega\odot\omega$ to ensure non-negativity of $\mathbf{w}$ and initialize sample weight $\omega_i=1$ for each sample $i$ and GNN's parameters $\theta$ with random uniform distribution. Once the initial values are given, in each iteration, we fix the sample weights $\omega$ and update the GNN's parameters $\theta$ by $\mathcal{L}_G$ with gradient descent, then compute the confounder weights $\alpha$ from the linear transform matrix $\mathbf{W}^{(K-1)}$. With $\alpha$ and  fixing the GNN's parameters $\theta$, we update the sample weights $\omega $ with gradient descent to minimize $\mathcal{L}_{DVD}(\mathbf{H}^{(K-1)})$. We iteratively update the sample weights $\mathbf{w}$ and GNN's parameters $\theta$ until $\mathcal{L}_G$ converges.

\subsection{Extension to GAT}
\label{sec::extend to GAT}
We can easily incorporate the VD/DVD term with other GNNs. We combine them with GAT and more extensions leave as future work. GAT utilizes an attention mechanism to aggregate neighbor information. It also follows the linear aggregation and transformation steps. Similar to GCN, the hidden embedding $\tilde{\mathbf{H}}^{(K-1)}$ is the input of VD/DVD term, and the variable weights $\alpha$ are calculated from the transformation matrix $\mathbf{W}^{(K-1)}$ and the sample weights $\mathbf{w}$ are used to reweight the softmax loss. Note that the original paper utilizes the same transformation matrix $\mathbf{W}^{(K-1)}$ for transforming embedding and learning attention values. Because $\alpha$ means the importance of each variable for classification, and it should be calculated from transformation matrix $\mathbf{W}^{(K-1)}$ for transforming embedding, hence we use separate matrices for transforming embedding and learning attention values, respectively. This modification does not change the performance of GAT in experiments.

\subsection{Complexity Analysis}
\par Compared with base models (e.g., GCN and GAT), the main incremental time cost is the complexity from VD/DVD term. For a training graph with $n$ labeled nodes, we analyze the time complexity of the VD/DVD term in each iteration. For calculating the VD loss, its complexity is $\mathcal{O}(np^2)$, where $p$ is the dimension of embeddings. And for DVD loss, its complexity is the same as VD, as the complexity of calculating variable weights $\alpha$ is $\mathcal{O}(np)$, which is relatively small comparing with $\mathcal{O}(np^2)$. For updating $\mathbf{w}$, the complexity is dominated by the step of calculating the partial gradients of the function $\mathcal{L}_{DVD}(\mathbf{H})$ with respect to variable $\mathbf{w}$. The complexity of $\frac{\partial \mathcal{L}_{DVD}(\mathbf{H})}{\partial \mathbf{w}}$ is $\mathcal{O}(np^2)$. In total, the complexity of each iteration for VD/DVD term in Algorithm~\ref{alg:GNN-DVD} is $\mathcal{O}(np^2)$. And it is quite smaller than the base models (e.g., the complexity of GCN is $\mathcal{O}({\mathcal{E}cp^2})$, where $\mathcal{E}$ is the number of edges and $c$ is the dimension of input node features).

\subsection{Discussion} In our paper, we propose to integrate two decorrelation terms (i.e., VD/DVD term) with GNN models to eliminate the estimation bias. Here we discuss the advantages and disadvantages of these two terms. VD term aggressively decorrelates all the variables learned by GNNs, however, it theoretically requires a large number of samples to achieve this goal. To overcome this dilemma, the DVD term is proposed to differentiate the variable weights in the VD term, aiming to remove the most unexpected correlation. However, due to the existence of unknown term $g(\mathbf{S})$ in Eq. (\ref{equ::METF_Y}), introducing more parameters to optimize may increase the instability of the model. Hence, when the number of labeled samples is large, performing GNN-VD may induce more stable results.

\section{Experiments}
\label{sec::Experiments}
\subsection{Datasets}
\label{sec:datasets}
\par Here, we validate the effectiveness of our methods on node classification with two kinds of selection bias, i.e., label selection bias and small sample selection bias. For label selection bias, we employ three widely used graph datasets: Cora, Citeseer and Pubmed~\cite{sen2008collective}. As in Section~\ref{sec::Experimental Investigation}, we get three biased degrees as well as the original unbiased labeled nodes for each dataset. For small sample selection bias, we conduct the experiments on NELL dataset~\cite{carlson2010toward}, where each class only has at most 1/5/10 labeled nodes for training. Due to the large scale of this dataset, the test nodes are easily to have distribution shifts from training nodes. The details of the datasets are summarized in Table~\ref{tab::data_stats}.
\begin{table*}[ht]
  \caption{Dataset statistics}
  \label{sample-table}
  \centering
  {
  \begin{tabular}{llllllll}
    \toprule
    \textbf{Dataset}   & \textbf{Type} & \textbf{Nodes} & \textbf{Edges} & \textbf{Classes} & \textbf{Features}  &\textbf{Bias degree ($\epsilon$)} & \textbf{Bias type}\\
    \midrule
    Cora      &Citation network   & 2,708 & 5,429 & 7& 1,433  & 0.7/0.8/0.9 & Label selection bias\\
    Citeseer  &Citation network       & 3,327 & 4,732 & 6& 3,703 & 0.7/0.8/0.9 & Label selection bias\\
    Pubmed    &Citation network  & 19,717 & 44,338 & 3& 500 &0.7/0.8/0.9& Label selection bias\\
    NELL     & Knowledge graph    & 65,755 & 266,144 & 210& 5,414 & 1/5/10 labeled nodes per class &Small sample selection bias\\
    \bottomrule
  \end{tabular}}
  \label{tab::data_stats}
\end{table*}
\subsection{Baselines}
We compare our proposed framework with several related baselines:
\begin{itemize}

\item  Base models: GCN~\cite{kipf2016semi} and GAT~\cite{velivckovic2017graph} are classical GNN methods. We utilize them as the base models in our framework, so they are the most related baselines to validate the effectiveness of the proposed framework.

\item GNM-GCN/GAT~\cite{zhou2019graph}: A GNN method which considers unbalanced label selection bias problem in transductive setting. They also utilize GCN/GAT as their base models.

\item Chebyshev~\cite{kipf2016semi}: It is a GCN-based method utilizing third-order Chebyshev filters.

\item SGC~\cite{pmlr-v97-wu19e}: It is a simplified GCN-based method, which reduces the excess complexity through successively removing nonlinearities and collapsing weight matrices between consecutive layers.
\item APPNP~\cite{klicpera_predict_2019}: It is one of the state-of-the-art GNN methods that combines PageRank with GCN.
\item Planetoid~\cite{yang2016revisiting}: It is a classical semi-supervised graph embedding method. We use its inductive variant.
\item MLP: It is a two-layer multilayer perceptron trained on the labeled nodes with only node features as input.
\item DGNN: It is the debiased GNN framework proposed in this paper. We incorporate the VD/DVD term with GCN/GAT under our proposed framework called GCN/GAT-VD/DVD.
\end{itemize}

\subsection{Experimental Setup}
As the Section~\ref{sec::Experimental Investigation} has described, for all datasets, to simulate the agnostic selection bias scenario, we first follow the inductive setting in~\cite{pmlr-v97-wu19e} that masks the validation and test nodes in the training phase and validation and test with the whole graph so that the test nodes will be agnostic. For GCN and GAT, we utilize the same two-layer architecture as their original paper~\cite{kipf2016semi,velivckovic2017graph}. We use the following sets of hyperparameters for GCN on Cora, Citeseer, Pubmed: 0.5 (dropout rate), $5\cdot 10^{-4}$ (L2 regularization) and 32 (number of hidden units); and for NELL: 0.1 (dropout rate), $1\cdot 10^{-5}$ (L2 regularization) and 64 (number of hidden units). For GAT on Cora, Citeseer, we use: 8 (first layer attention heads), 8 (features each head), 1 (second layer attention head), 0.6 (dropout),  $5\cdot 10^{-4}$ (L2 regularization); and for Pubmed: 8 (second layer attention head),  $1\cdot10^{-3}$ (L2 regularization), other parameters are the same as Cora and Citeseer. To fair comparison, the GNN part of our model uses the same architecture and hyper-parameters with the base model and we grid search $\lambda_1$ and $\lambda_2$ from $\{0.01, 0.1, 1, 10, 100\}$. For other baselines, we use the optimal hyper-parameters in the original literatures on each dataset. For all the experiments, we run each model 10 times with different random seeds and report its average Accuracy results.


\begin{table*}
\centering
\caption{Performance of three citation networks. The `*' indicates the best results of the baselines. Best results of all methods are indicated in bold. `\% gain over GCN/GAT' means the improvement percent of GCN/GAT-DVD against GCN/GAT.}
\resizebox{.98\textwidth}{!}{
			\begin{tabular}{|c||c|c|c|c|c|c|c|c|c|c|c|c|}
				\hline
				\multirow{2}{*}{Method}&
				\multicolumn{4}{c|}{\textbf{Cora}}&\multicolumn{4}{c|}{\textbf{Citeseer}}&\multicolumn{4}{c|}{\textbf{Pubmed}} \cr\cline{2-13}
				& Unbiased	& Light            & Medium                         & Heavy        & Unbiased   & Light           & Medium          & Heavy      &Unbiased   & Light     & Medium    & Heavy                     \cr\cline{2-13}
				\hline
				\hline
				
				MLP    &0.5296  & 0.5624          & 0.5197                      & 0.5087        &0.5438  & 0.4532          & 0.3757          & 0.3893     &0.6914    & 0.6852          & 0.6620   & 0.6378
				\cr
				Planetoid~\cite{yang2016revisiting}   &0.6650   & 0.5890          & 0.5240                  & 0.5180      &0.6720    & 0.5160          & 0.5140          & 0.4880     &0.744     & 0.7160          & 0.6770   & 0.6680
					\cr
				Chebyshev~\cite{defferrard2016convolutional}  & 0.7407   & 0.7116          & 0.7006                  & 0.6809       &$0.7232^*$   & 0.6542          & 0.6276          & 0.5920      &0.7450    & 0.7358          & 0.6862   & 0.6732
					\cr
				SGC~\cite{pmlr-v97-wu19e}    &0.779  & 0.7800          & 0.7800                  & 0.7530    &0.724      & 0.6780          & $0.6730^*$          & 0.6200      &0.781    & $\mathbf{0.7880}^*$          & 0.7560   & 0.6800
					\cr
				APPNP~\cite{klicpera_predict_2019}   &$0.8132^*$   &0.7913          & 0.7689                  & 0.7629    &0.6862     & 0.6478          & 0.6052        & 0.5903     &0.7731    & 0.7639         & 0.7369   & 0.6862
                \cr\hline
                GNM-GCN~\cite{zhou2019graph} &0.7594 &0.7423         & 0.7531                  & 0.7196      &0.6054   & 0.5793          & 0.5717        & 0.5125    &0.7654     & 0.7552         & 0.7381   & 0.7072\cr
                GNM-GAT~\cite{zhou2019graph} &0.7976 &0.7875         & 0.7638                  & 0.7404   &0.6832      & 0.6524          & 0.6487        & 0.5865     &   0.7666 & 0.7438        & 0.7568   & 0.6891
				
				\cr\hline
				GCN~\cite{kipf2016semi}  &0.7909    & 0.7851          & 0.7775               & 0.7422      &0.7075   & 0.6786         & 0.5952         & 0.5551       &$0.7845^*$      & 0.7673          & 0.7545   & $0.7247^*$
				\cr
				GCN-VD  &0.7980   & 0.7951          & 0.7855               & 0.7522       &0.7122   & 0.6844     & 0.6676 & 0.6408    &\textbf{0.7888}     & 0.7727 & 0.7729         & 0.7399
				
				\cr
				GCN-DVD  &0.7951   & 0.7959        & 0.7885             & 0.7555   &   0.7128   & 0.6908     & 0.6769         & 0.6496    &0.7874    & 0.7741
 & \textbf{0.7746}        & \textbf{0.7542}

	            \cr
	            \% gain over GCN &0.53\% & 1.38\%        & 1.41\%              & 1.79\%       &0.75\%  & 1.8\%    & 14.2\%         & 17.0\%     &0.37\%    & 0.89\%
 & 2.67\%        & 4.07\%
			\cr\hline 	GAT~\cite{velivckovic2017graph}  &0.8100    & $0.8067^*$          & $0.8019^*$               & 0.7578       &0.7224  & $0.7033^*$         & 0.6683      &$0.6475^*$   & 0.7714              & 0.7665          & $0.7579^*$   & 0.7068
				\cr
				GAT-VD    &0.8133 & 0.8146          & 0.8079               & \textbf{0.7708}      &0.7288    & 0.7149     & \textbf{0.6833}             & 0.6611   &0.7732       & 0.7783 & 0.7689         & 0.7149
				\cr
				GAT-DVD  &\textbf{0.8139}   & \textbf{0.8179}          & \textbf{0.8119}               & 0.7694       &\textbf{0.7294}   & \textbf{0.7172}     & 0.6825             & \textbf{0.6627}   &0.7735       & 0.7788 &0.7723         & 0.7210
				
				\cr
	            \% gain over GAT  & 0.48\% & 1.39\%        & 1.26\%              & 1.53\%      &0.97\%   & 1.97\%    & 2.12\%         & 2.34\%       &0.27\%  & 1.6\%
 & 1.9\%        & 2.0\%
			\cr\hline					
			\end{tabular}}
	\label{tab:performance}
\end{table*}
\begin{table*}[ht]
  \caption{Performance of NELL. The `*' indicates the best results of the baselines. Best results of all methods are indicated in bold. `Improvement' means the improvement percent of GCN-VD/DVD (selected better results) against GCN.}
  \label{sample-table}
  \centering{
  \begin{tabular}{llllllll}
    \cr\hline
    Dataset   & MLP & Planetoid & SGC  & GCN  &GCN-VD & GCN-DVD & Improvement
    \cr\hline
    NELL-1       & 0.2385 &0.3901 & 0.4128  & $0.4416^*$ & 0.4652  & \textbf{0.4734} & 7.2\%
    \cr
    NELL-5      &0.4938   & 0.3519
 & 0.6295
 &  $0.7030^*$ & \textbf{0.7424}  & 0.7361 & 5.6\%
    \cr
    NELL-10      &  0.5838 & 0.5149 & 0.6275 & $0.7615^*$ & \textbf{0.7734}
  & 0.7727  &1.6\%
    \cr\hline
  \end{tabular}}
  \label{tab::NELL}
\end{table*}

\subsection{Results on Label Selection Bias Datasets}
\par The results are given in Table~\ref{tab:performance}, and we have the following observations. First, the proposed models (i.e., GCN/GAT with VD/DVD terms) always achieve the best performances in most cases, which well demonstrates that the effectiveness of our proposed debiased GNN framework. Second, comparing with base models, our proposed models all achieve up to 17.0\%
performance improvements, and gain larger improvements under heavier bias scenarios. Since the major difference between our model with base models is the VD/DVD regularizer, we can safely attribute the significant improvements to the effective decorrelation term and its seamless joint with GNN models. Third, GCN/GAT-DVD achieves better results than GCN/GAT-VD in most cases. It validates the importance and effectiveness of differentiating variables' weights in the semi-supervised setting. Moreover, our model still outperforms baselines in the unbiased setting. In real applications, it is hard to control the collection process without any distribution shift from the training set to the test set~\cite{huang2006correcting}. Therefore, the problem we study is ubiquitous in reality and our method is effective in most scenarios.

\subsection{Results on Small Sample Selection Bias Datasets}
As NELL is a large-scale graph, we cannot run GAT on a single GPU with 16GB memory. We only perform GCN-VD/DVD and compare with representative methods which can perform on this dataset.
The results are shown in Table~\ref{tab::NELL}. First, GCN-VD/DVD achieves significant improvements over GCN. It indicates that selection bias could be induced by a small number of labeled nodes and our proposed method relieve the estimation bias. Moreover, with fewer labeled nodes, i.e., larger selection bias, our methods achieve larger improvements over base models. It further validates our method is an effective method against heavy bias. Moreover, GCN-DVD further improves GCN-VD with a large margin on NELL-1 dataset. It means that decorrelating all the variable pairs equally is suboptimal, and our differentiated strategy is effective when labeled nodes are scarce. With the number of labeled nodes increases, it may not necessary to differentiate the variable weights, but GCN-DVD achieves competitive results with GCN-VD and still outperforms the base model with a clear margin.

\begin{figure*}[!htbp]
\centering
\subfigure[Cora]{
\includegraphics[ width=4cm]{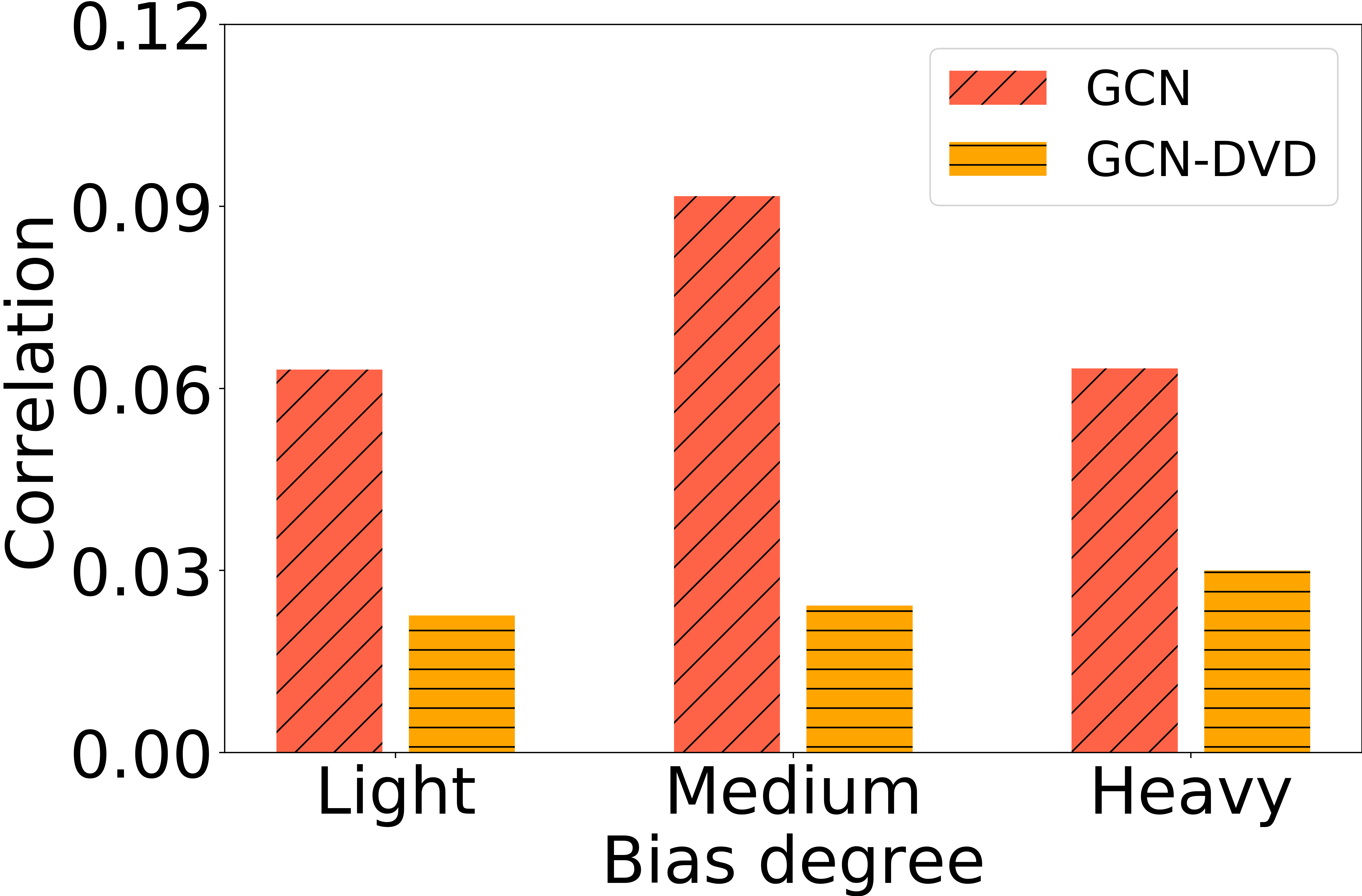}
\label{fig:gradient}
}
\subfigure[Citeseer]{
\includegraphics[ width=4cm]{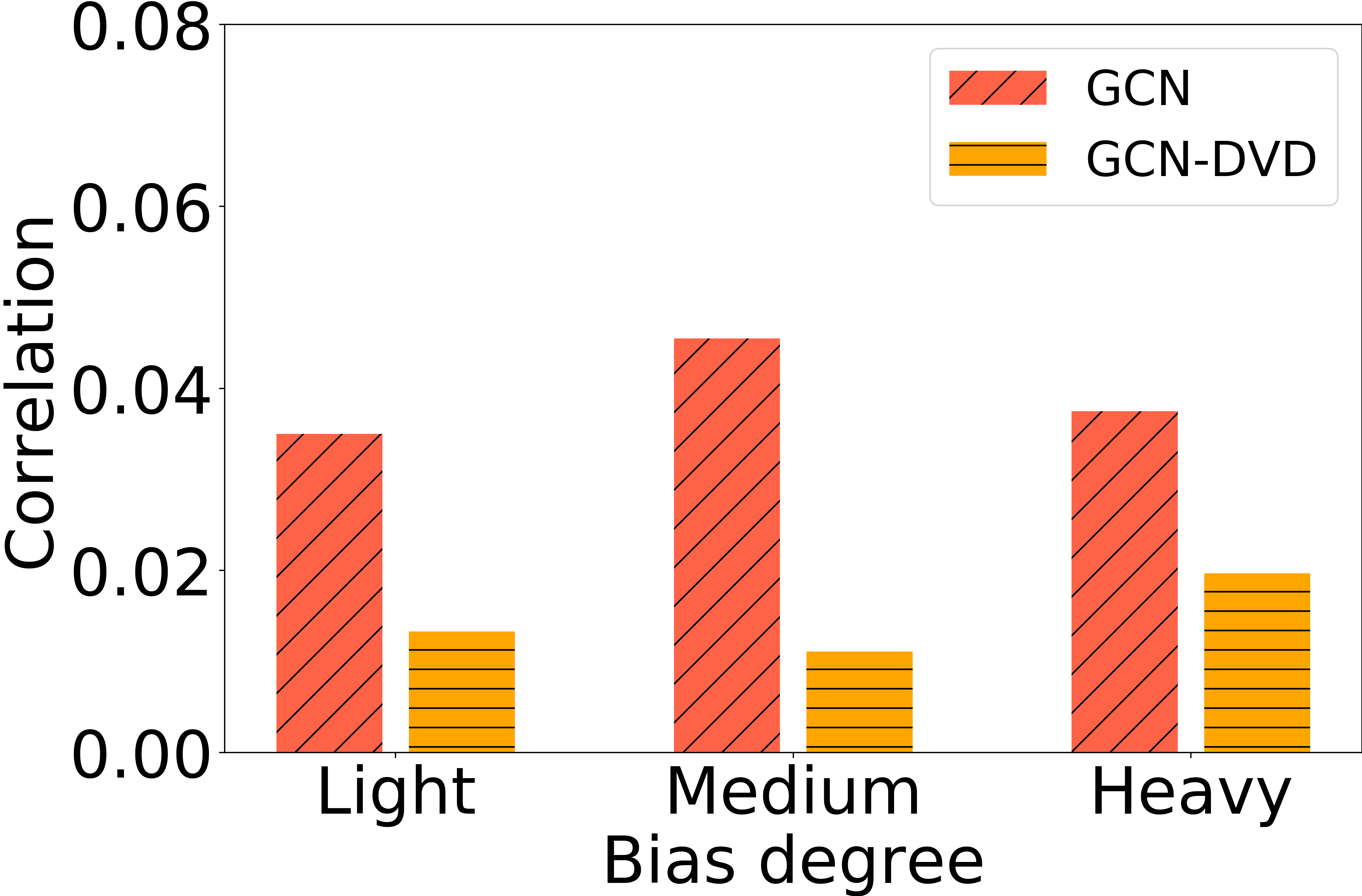}
\label{fig:gradient}
}
\subfigure[Pubmed]{
\includegraphics[width=4cm]{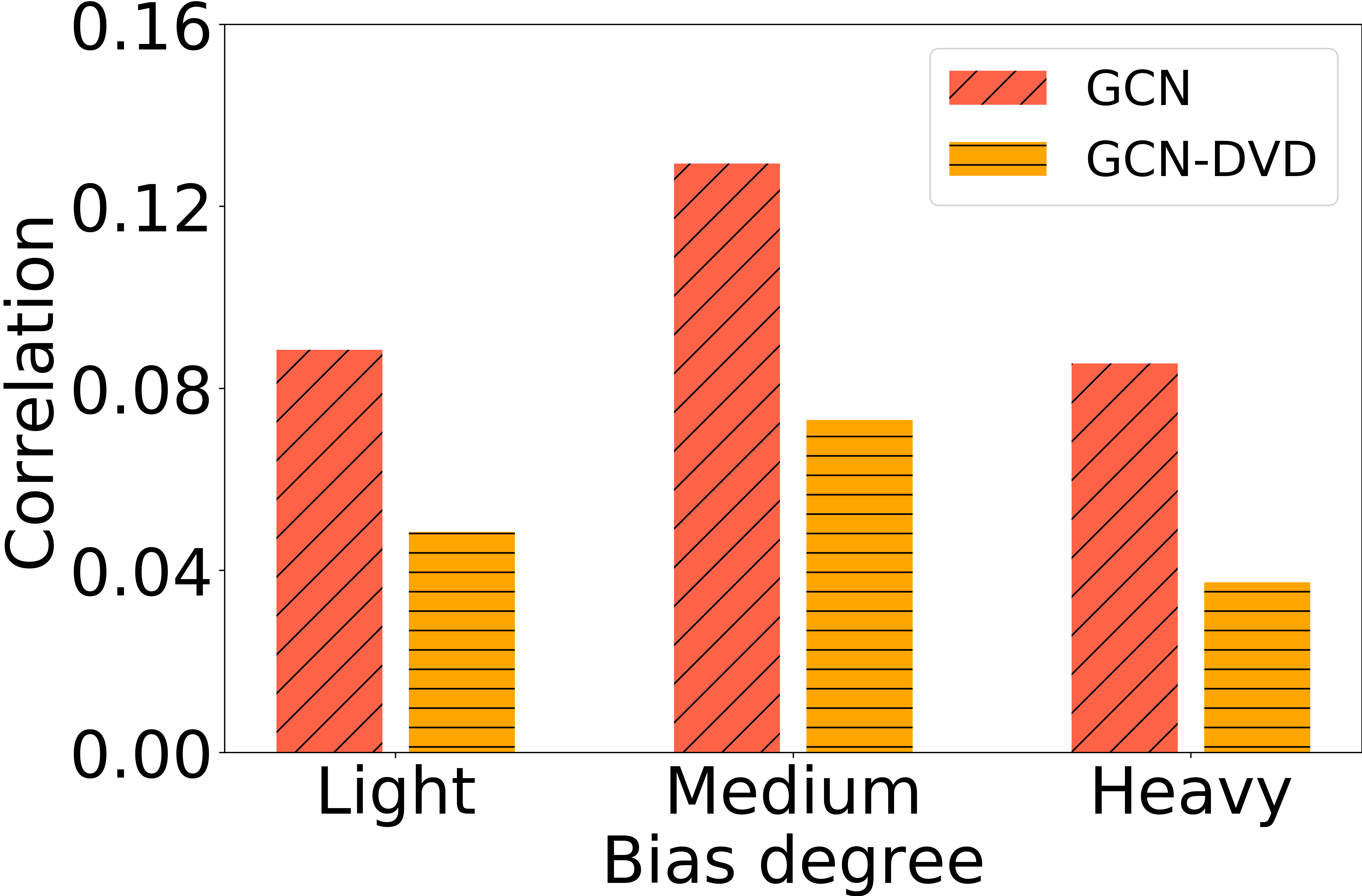}
\label{fig:gradient}
}

\caption{Embedding correlation analysis on unweighted and weighted GCN.}
\label{fig:correlation}
\end{figure*}
\subsection{Sample Weight Analysis}
Here we analyze the effect of sample weights $\mathbf{w}$. We compute the amount of correlation in the labeled nodes' embeddings $\tilde{\mathbf{H}}^{(K-1)}$ learned by standard GCN and the weighted embeddings of the same layer learned by GCN-DVD. Note that, the weights are the last iteration of sample weights of GCN-DVD.  Following~\cite{cogswell2015reducing,wang2020decorrelated}, the amount of correlation of GCN and GCN-DVD is measured by Frobenius norm of cross-corvairance matrix $||C||_F^2$ computed from variables of $\tilde{\mathbf{H}}^{(K-1)}$ and weighted $\tilde{\mathbf{H}}^{(K-1)}$ respectively, where $C_{ij}$ represents the covariance between pairwise variable $i$ and $j$ and the main diagonal of $C$ is set as zero vector.  Figure~\ref{fig:correlation} shows the amount of correlation in unweighted and weighted embeddings, and we observe that the embeddings' correlations in all datasets are reduced, indicating that the weights learned by GCN-DVD can reduce the correlation between embedded variables. Moreover, as it is hard to reduce the correlation to zero, the necessity of differentiating variables' weights is further validated.

\begin{figure*}[!htbp]
\centering

\subfigure[Light]{
\includegraphics[height=1.3in, width=1.8in]{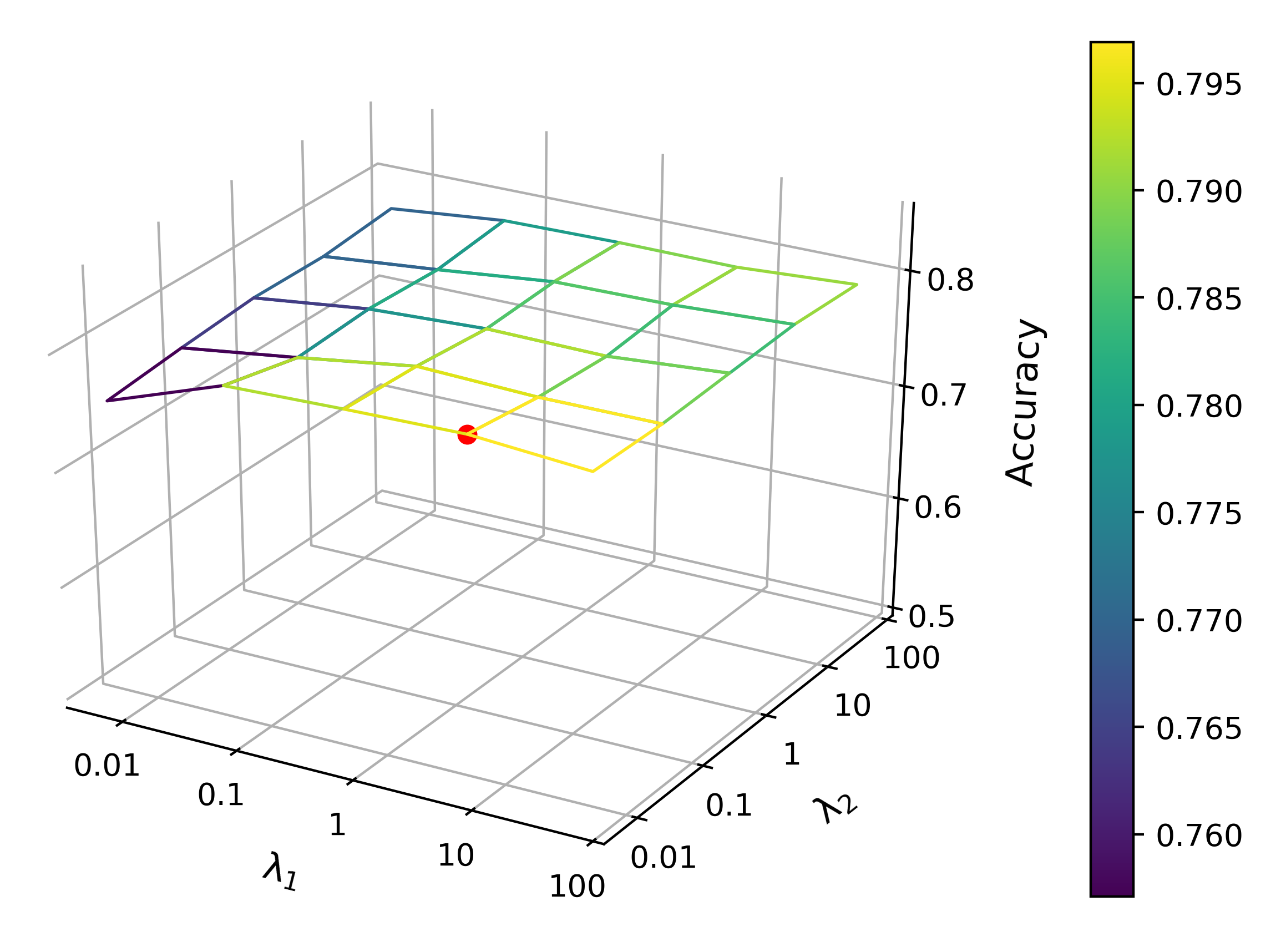}
\label{fig:gradient}
}
\subfigure[Medium]{
\includegraphics[height=1.3in, width=1.8in]{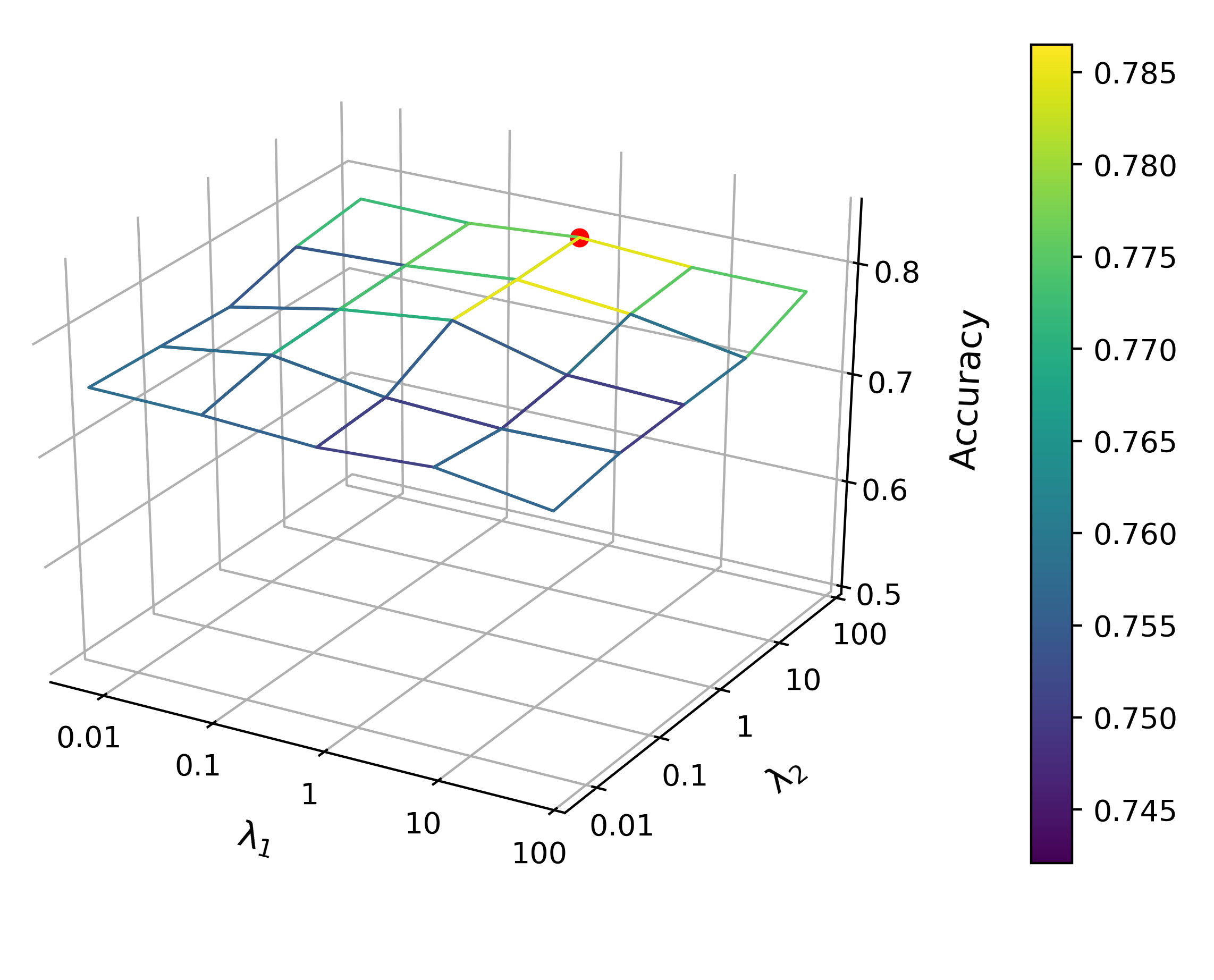}
\label{fig:gradient}
}
\subfigure[Heavy]{
\includegraphics[height=1.3in, width=1.8in]{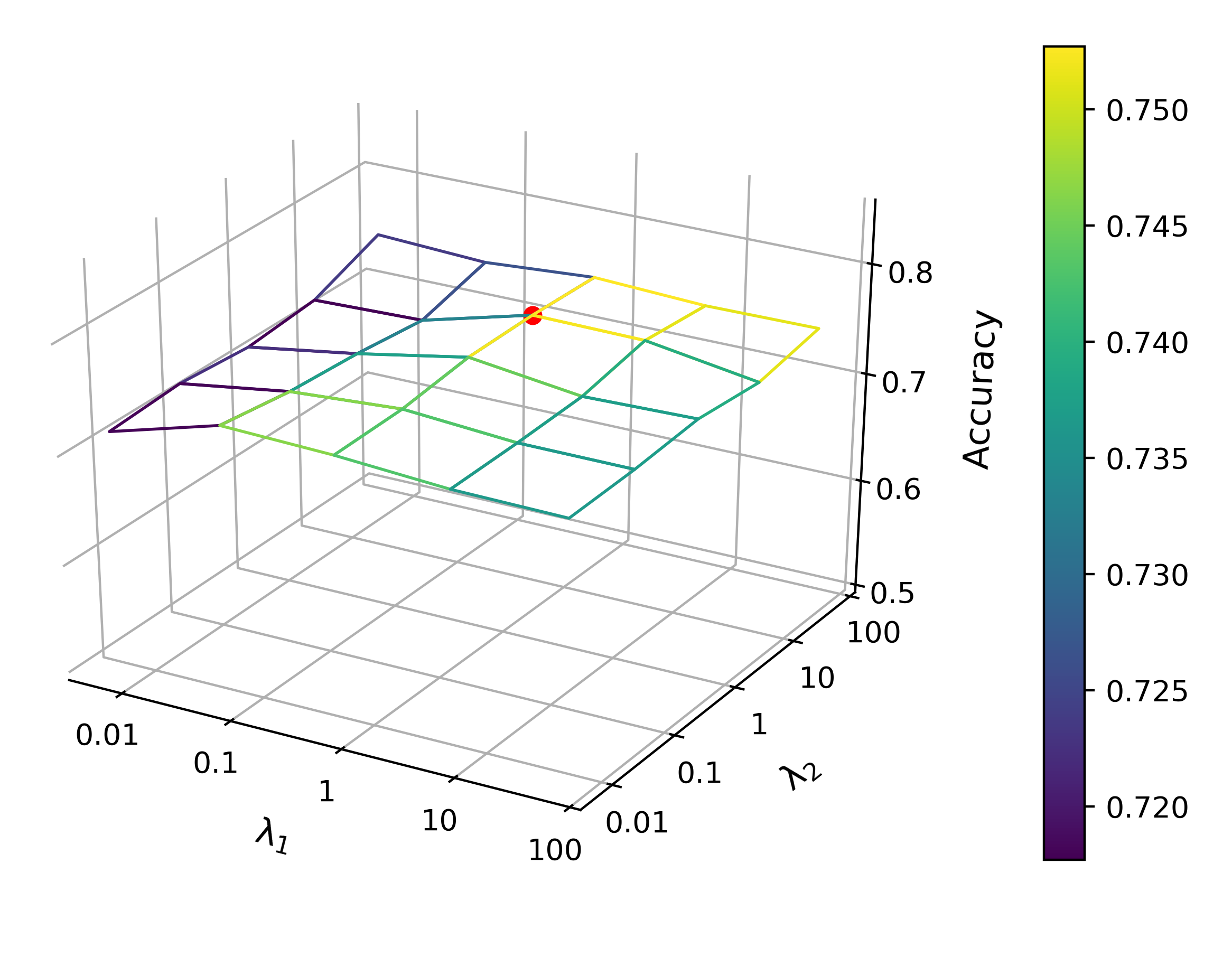}
\label{fig:123}
}
\caption{Accuracy of GCN-DVD with different $\lambda_1$ and $\lambda_2$ on different biased Cora datasets.}
\label{fig:para}
\end{figure*}

\begin{figure*}[!htbp]
\centering

\subfigure[Light]{
\includegraphics[height=1.3in, width=1.8in]{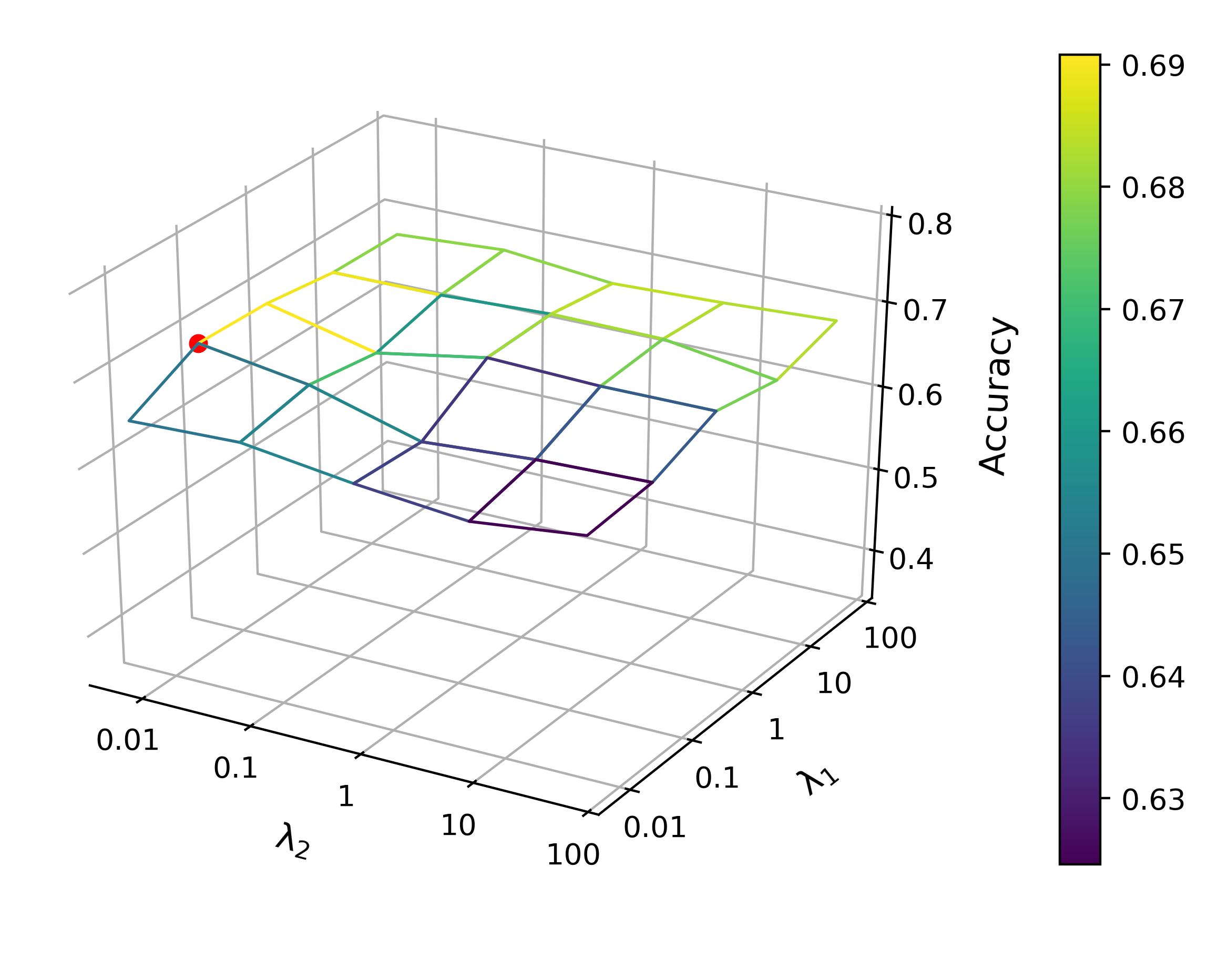}
\label{fig:gradient}
}
\subfigure[Medium]{
\includegraphics[height=1.3in, width=1.8in]{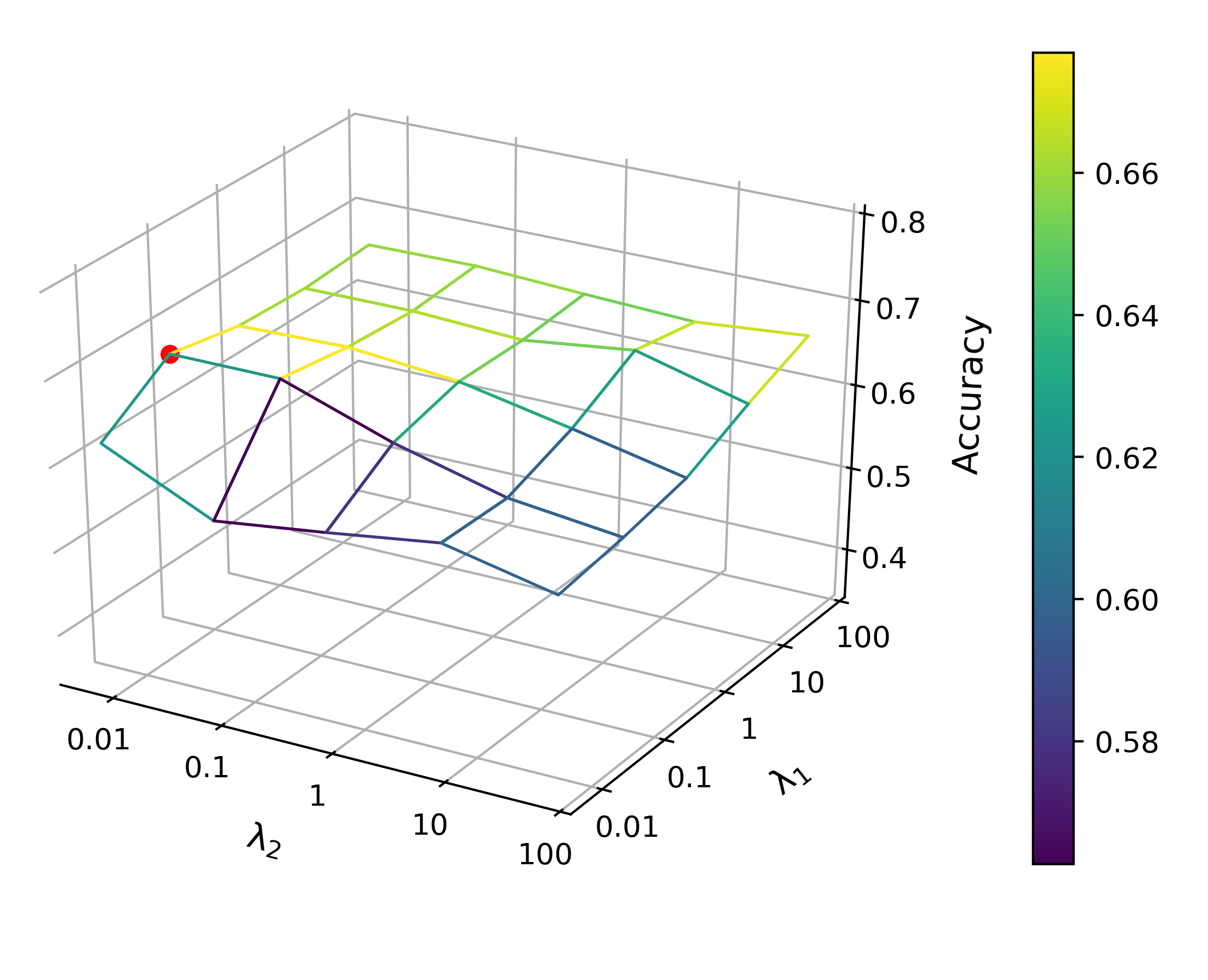}
\label{fig:gradient}
}
\subfigure[Heavy]{
\includegraphics[height=1.3in, width=1.8in]{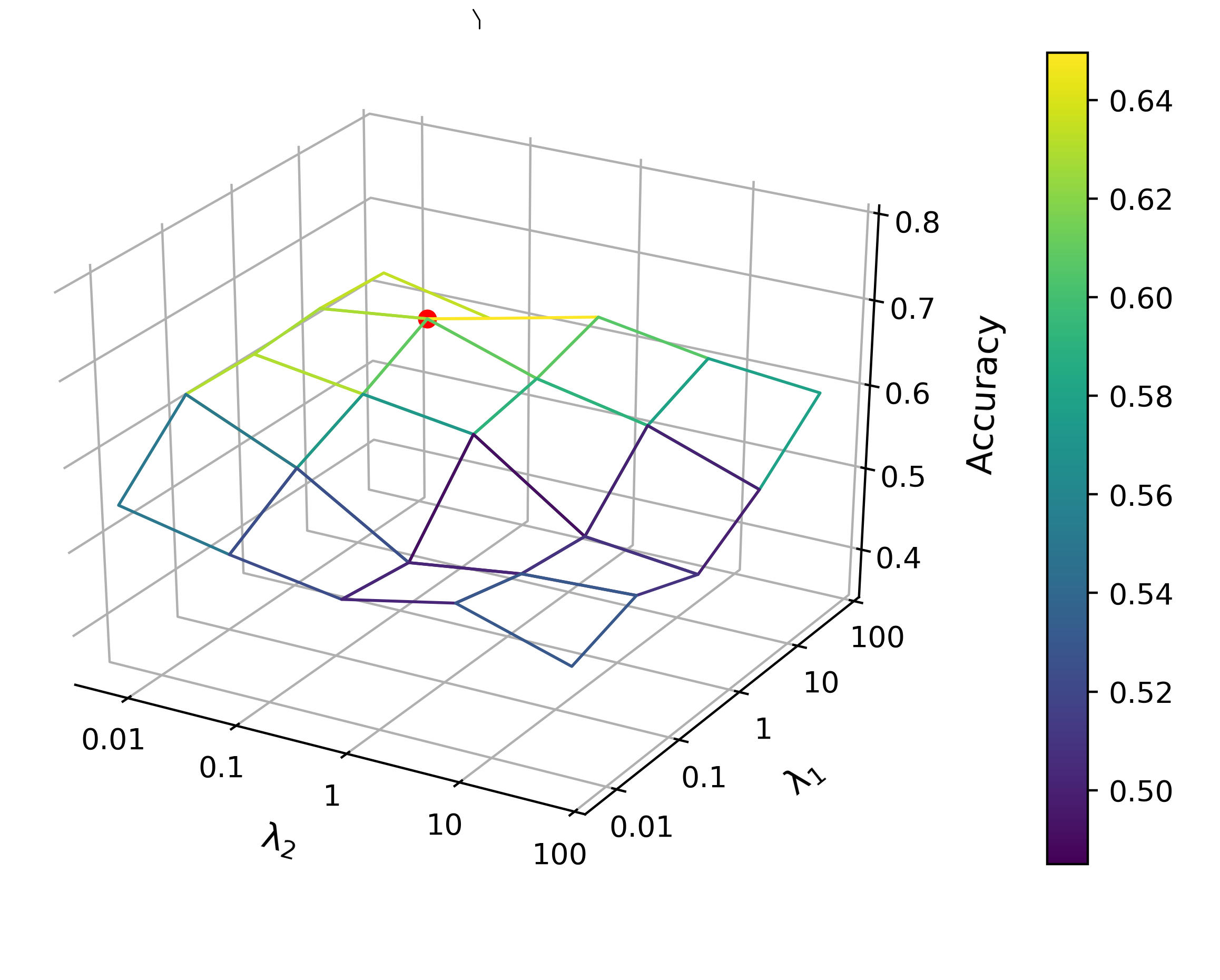}
\label{fig:123}
}
\caption{Accuracy of GCN-DVD with different $\lambda_1$ and $\lambda_2$ on different biased Citeseer datasets.}
\label{fig:para_citeseer}
\end{figure*}

\begin{figure*}[!htbp]
\centering

\subfigure[Light]{
\includegraphics[height=1.3in, width=1.8in]{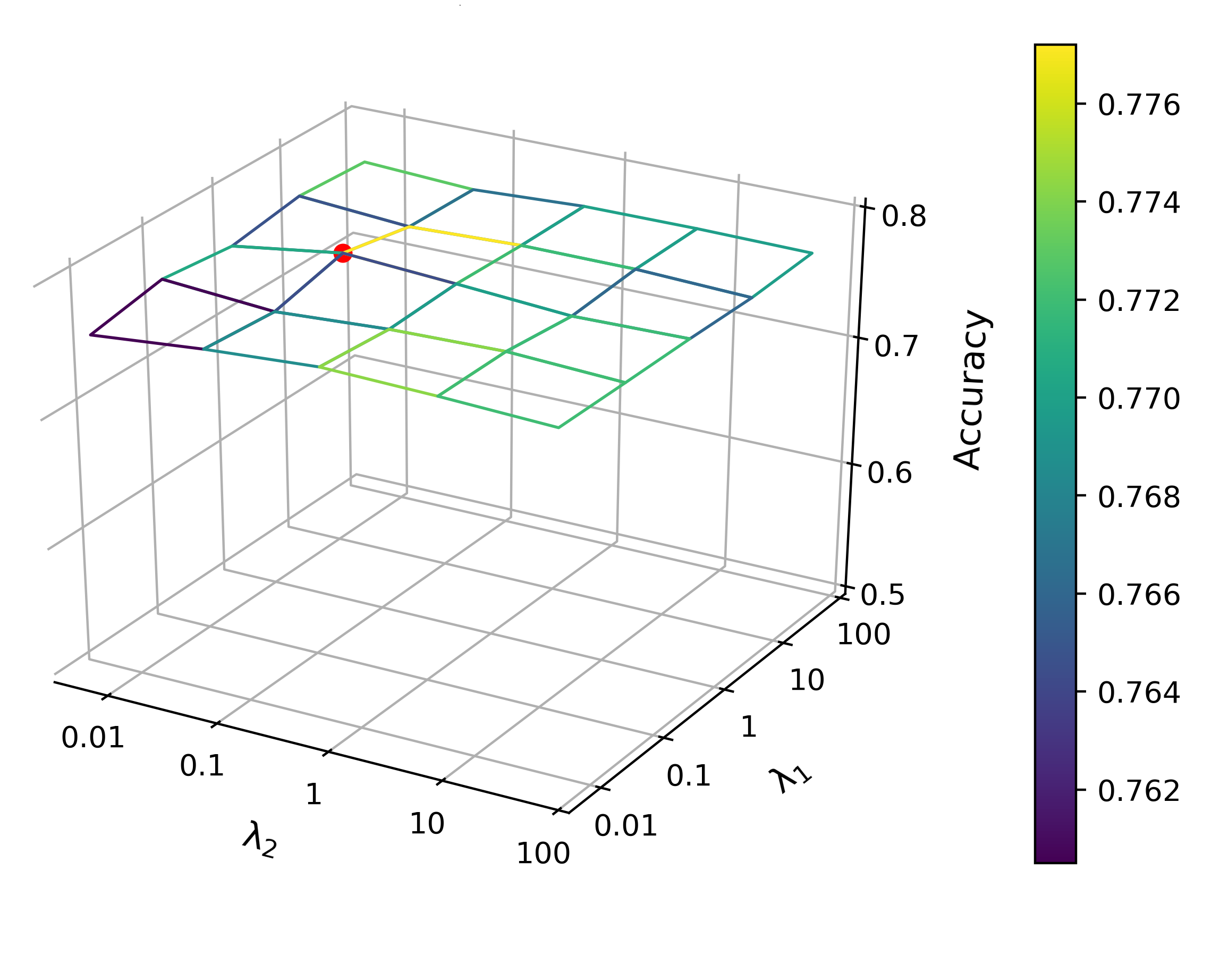}
\label{fig:gradient}
}
\subfigure[Medium]{
\includegraphics[height=1.3in, width=1.8in]{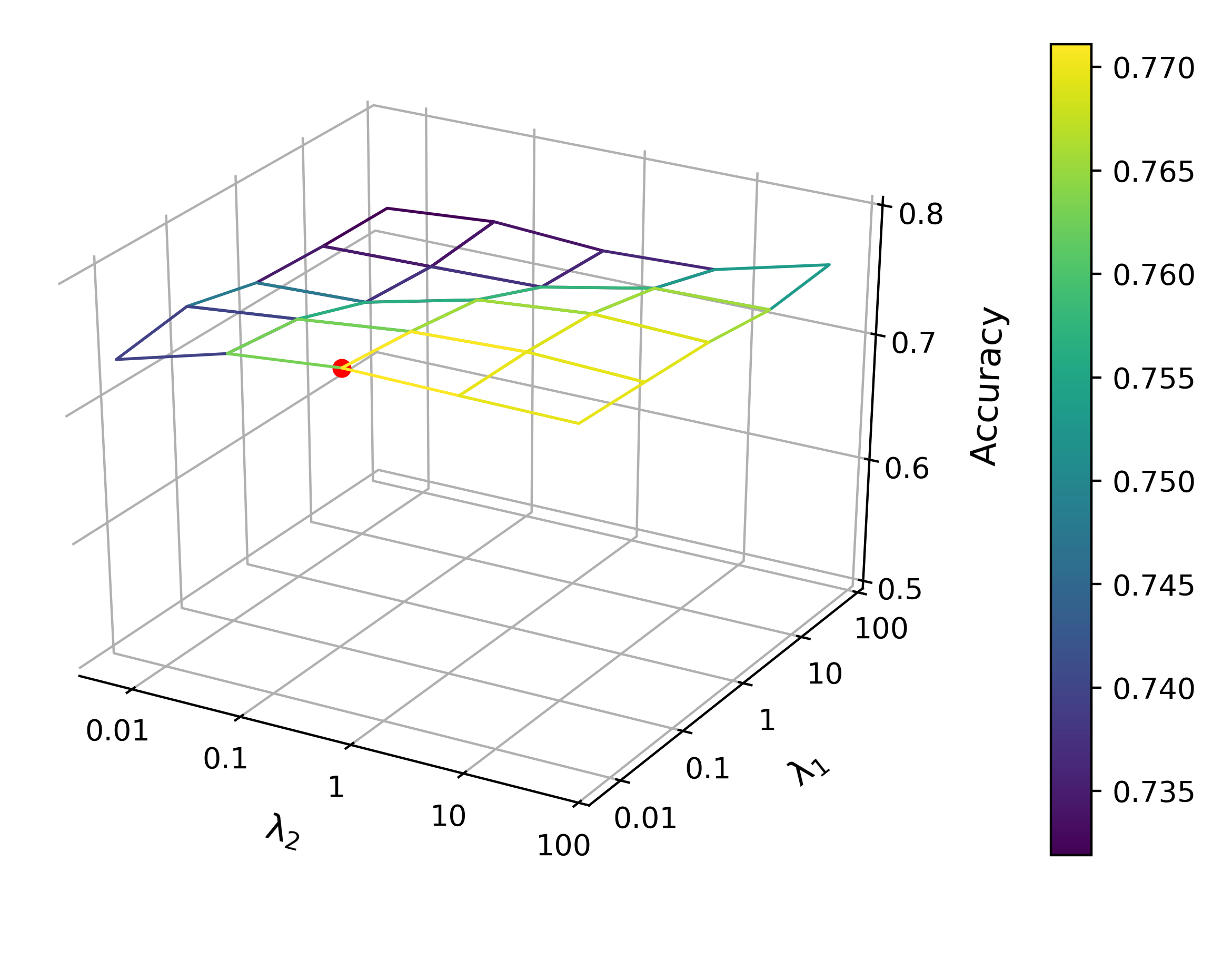}
\label{fig:gradient}
}
\subfigure[Heavy]{
\includegraphics[height=1.3in, width=1.8in]{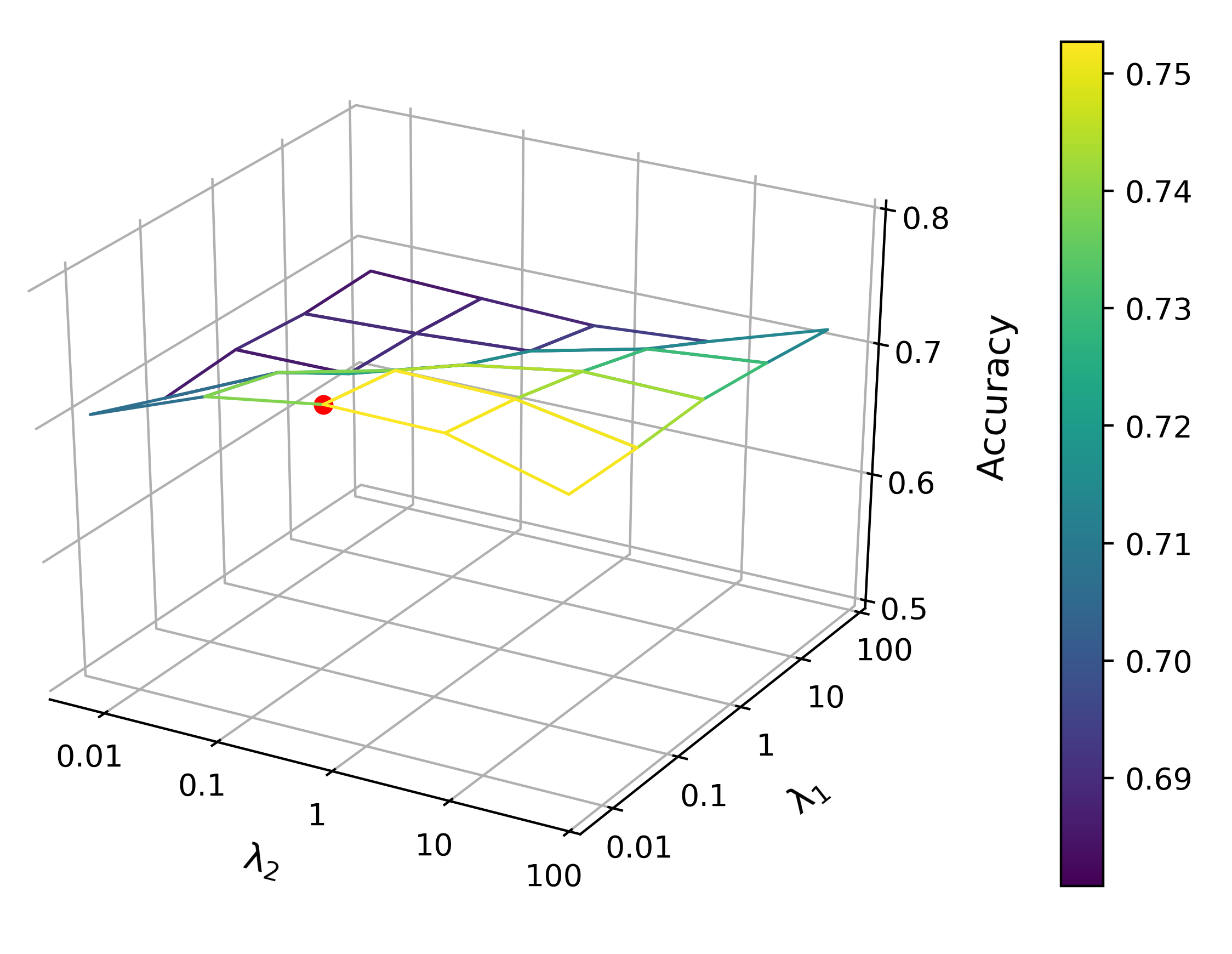}
\label{fig:123}
}
\caption{Accuracy of GCN-DVD with different $\lambda_1$ and $\lambda_2$ on different biased Pubmed datasets.}
\label{fig:para_pubmed}
\end{figure*}

\subsection{Parameter sensitivity}
We study the sensitiveness of parameters and report the results of GCN-DVD on three citation networks in Fig.~\ref{fig:para}-\ref{fig:para_pubmed}. The experimental results show that GCN-DVD is relatively stable to $\lambda_1$ and $\lambda_2$ with wide ranges in most cases, indicating the robustness of our model.

\subsection{Training time per epoch}
We report the results for the mean training time of GCN and GCN-DVD per epoch (forward pass, cross-entropy calculation, backward pass) for 200 epochs on Cora, Citeseer and Pubmed datasets, measured in seconds wall-clock time, in Table~\ref{tab::time}. These methods are performed on a RTX 3090 GPU Card. As we can see, the training time of GCN-DVD term has the same order of magnitude with GCN. More importantly, the training time of the DVD term will not be influenced by the base model we choose, i.e., when the base model is GAT, the running time of DVD term will not change.

\begin{table}[ht]
  \caption{The training time per epoch.}
  \label{sample-table}
  \centering{
  \begin{tabular}{llll}
    \cr\hline
       & Cora & Citeseer & Pubmed 
    \cr\hline
    GCN       & $1.29\times 10^{-2}$ & $2.00\times 10^{-2}$ & $1.11\times 10^{-1}$
    \cr
    GCN-DVD      &$6.19\times 10^{-2}$   & $8.46\times 10^{-2}$
 & $2.42\times 10^{-1}$
    \cr\hline
  \end{tabular}}
  \label{tab::time}
\end{table}
\section{Related Works}
\label{sec::related work}
\par In the past few years, Graph Neural Networks (GNNs)~\cite{scarselli2008graph,kipf2016semi,velivckovic2017graph,klicpera_predict_2019,xu2018how,jin2021bite,fan2019metapath,fan2020one2multi,bai2020learning} have become the major technology to capture patterns encoded in the graph due to its powerful representation capacity. Recently, KPGNN~\cite{cao2021knowledge} applies GNNs to the social event detection task by preserving the incremental knowledge emerging in social data. MRFasGCN~\cite{jin2019graph} integrates GCN with a Markov Random Fields (MRF) model to deal with the semi-supervised community detection problem. Not only pursuing the performance of GNNs on clean data, \cite{lin2020exploratory} proposes an exploratory adversarial attack method, called EpoAtk, to test whether existing GNNs are robust with adversarial perturbations on graphs. Although the current GNNs have achieved great success, when applied to the inductive setting, they all assume that training nodes and test nodes follow the same distribution. However, this assumption does not always hold in real applications. GNM~\cite{zhou2019graph} first pays attention to the label selection problem on graph learning, and it learns an IPW estimator to estimate the probability of each node to be selected and uses this probability to reweight the labeled nodes. However, it heavily relies on the accuracy of the IPW estimator, which depends on the label assignment distribution of the whole graph, hence it is more suitable for the transductive setting.
\par To enhance the stability in unseen varied distributions, some literatures~\cite{kuang2020stable,shen2020sample} have revealed the connection between correlation and prediction stability under model misspecification. Moreover, a kind of literatures~\cite{ma2016decorrelation,rodriguez2016regularizing,zhang2018removing} have studied the problem of removing the features correlation effect in neural networks, which brings great benefits for deep neural networks. However, these methods are built on simple regressions or regular neural networks such as CNNs, but GNNs have more complex architectures and properties needed to be considered. We also notice that \cite{shen2020stable} propose a differentiated variable decorrelation term for linear regression. However, this decorrelation term requires multiple environments with different correlations between stable variables and unstable variables available in the training stage while our method does not require this prior knowledge.

\section{Conclusion}
In this paper, we investigate a general and practical problem: learning GNNs with agnostic label selection bias. The selection bias will inevitably cause the GNNs to learn the biased correlation between aggregation mode and class label and make the prediction unstable. We propose a novel debiased GNN framework, which combines the decorrelation technique with GNNs in a unified framework. Extensive experiments well demonstrate the effectiveness and flexibility of DGNN.

\section*{Acknowledgment}
This work is partially supported by the National Natural Science Foundation of China (No. U20B2045, 62192784, 62172052, 61772082, 62002029) and the Fundamental Research Funds for the Central Universities (No. 2021RC28). Kun Kuang is supported in part by National Natural Science Foundation of China (No. 62006207), Young Elite Scientists Sponsorship Program by CAST and Zhejiang Province Natural Science Foundation (No. LQ21F020020). Shaohua Fan is supported by BUPT Excellent Ph.D. Students Foundation (No. CX2021311) and China Scholarship Council. We thank Dr. Tianchi Yang for discussion on the proof of Theorem 3.


%

\appendices
\section{Proof of Theorem 2}
\label{appendix::Theorem 2}
\begin{equation}
\begin{aligned}
\hat{\mathbf{w}}&=\arg\min_\mathbf{w}\sum^p_{j=1}(\alpha^\mathrm{T} \text{abs}(\mathbf{H}_{.j}^\mathrm{T}\Lambda_\mathbf{w}\mathbf{H}_{.-j}/n\\
&-\mathbf{H}_{.j}^\mathrm{T} \mathbf{w}/n\cdot\mathbf{H}_{.-j}^\mathrm{T} \mathbf{w}/n))^2+ \frac{\lambda_1}{n}\sum_{i=1}^{n}\mathbf{w}_i^2 + \lambda_2 (\frac{1}{n}\sum_{i=1}^n\mathbf{w}_i-1)^2
    \label{equ:1}
\end{aligned}
\end{equation}
\begin{proof} For simplicity, we denote $\mathcal{L}_1=\sum^p_{j=1}(\alpha^\mathrm{T}\cdot \text{abs}(\mathbf{H}_{.j}^\mathrm{T}\Lambda_\mathbf{w}\mathbf{H}_{.-j}/n-\mathbf{H}_{.j}^\mathrm{T} \mathbf{w}/n\cdot\mathbf{H}_{.-j}^\mathrm{T} \mathbf{w}/n))^2$,
$\mathcal{L}_2=\frac{1}{n}\sum_{i=1}^{n}\mathbf{w}_i^2$, $\mathcal{L}_3=(\frac{1}{n}\sum_{i=1}^n\mathbf{w}_i-1)^2$ and $\mathcal{F}(\mathbf{w})=\mathcal{L}_1 + \lambda_1\mathcal{L}_1+\lambda_2\mathcal{L}_2$.
We first calculate the Hessian matrix of $\mathcal{F}(\mathbf{w})$, denoted as $\mathbf{H}_e$, to prove the uniqueness of the optimal solution $\hat{\mathbf{w}}$, as follows:
\begin{equation}
    \mathbf{H}_e=\frac{\partial^2\mathcal{L}_1}{\partial\mathbf{w}^2}+
    \lambda_1\frac{\partial^2\mathcal{L}_2}{\partial\mathbf{w}^2}+
    \lambda_2\frac{\partial^2\mathcal{L}_3}{\partial\mathbf{w}^2}
    \nonumber
\end{equation}
For the term $\mathcal{L}_1$, we can rewrite it as:
\begin{equation}
    \begin{aligned}
\mathcal{L}_1&=\sum_{j\neq k}\alpha_i^2\alpha_k^2(\frac{1}{n}\sum_{i=1}^n\mathbf{H}_{i,j}\mathbf{H}_{i,k}\mathbf{w}_i\\
&-(\frac{1}{n}\sum_{i=1}^n\mathbf{H}_{i,j}\mathbf{w}_i)(\frac{1}{n}\sum_{i=1}^n\mathbf{H}_{i,k}\mathbf{w}_i))^2\\
&=\sum_{j\neq k}\alpha_i^2\alpha_k^2((\frac{1}{n}\sum_{i=1}^n\mathbf{H}_{i,j}\mathbf{H}_{i,k}\mathbf{w}_i)^2\\
&-(\frac{2}{n}\sum_{i=1}^n\mathbf{H}_{i,j}\mathbf{H}_{i,k}\mathbf{w}_i)(\frac{1}{n}\sum_{i=1}^n\mathbf{H}_{i,j}\mathbf{w}_i)(\frac{1}{n}\sum_{i=1}^n\mathbf{H}_{i,k}\mathbf{w}_i)\\
&+ ((\frac{1}{n}\sum_{i=1}^n\mathbf{H}_{i,j}\mathbf{w}_i)(\frac{1}{n}\sum_{i=1}^n\mathbf{H}_{i,k}\mathbf{w}_i))^2)
\end{aligned}
\nonumber
\end{equation}
And when $|\mathbf{H}_{i,j}|\leq c$, for any variable $j$ and $k$, and $|\mathbf{w}_{i}|\leq c$, we have $\frac{\partial^2}{\partial\mathbf{w}^2}(\frac{1}{n}\sum_{i=1}^n\mathbf{H}_{i,j}\mathbf{H}_{i,k}\mathbf{w}_i)^2=\mathcal{O}(\frac{1}{n^2})$, $\frac{\partial^2}{\partial\mathbf{w}^2}(\frac{1}{n}\sum_{i=1}^n\mathbf{H}_{i,j}\mathbf{w}_i)(\frac{1}{n}\sum_{i=1}^n\mathbf{H}_{i,k}\mathbf{w}_i)=\mathcal{O}(\frac{1}{n^2})$ and $\frac{\partial^2}{\partial\mathbf{w}^2}((\frac{2}{n}\sum_{i=1}^n\mathbf{H}_{i,j}\mathbf{H}_{i,k}\mathbf{w}_i)(\frac{1}{n}\sum_{i=1}^n\mathbf{H}_{i,j}\mathbf{w}_i)(\frac{1}{n}\sum_{i=1}^n\mathbf{H}_{i,k}\mathbf{w}_i))=\mathcal{O}(\frac{1}{n^2})$. Then with $|\alpha_i|\leq c$, we have $\alpha_i^2\alpha_k^2\frac{\partial^2}{\partial\mathbf{w}^2}(\frac{1}{n}\sum_{i=1}^n\mathbf{H}_{i,j}\mathbf{H}_{i,k}\mathbf{w}_i-(\frac{1}{n}\sum_{i=1}^n\mathbf{H}_{i,j}\mathbf{w}_i)(\frac{1}{n}\sum_{i=1}^n\mathbf{H}_{i,k}\mathbf{w}_i))^2=\mathcal{O}(\frac{1}{n^2})$. $\mathcal{L}_1$ is sum of $p(p-1)$ such terms. Then we have
\begin{equation}
    \frac{\partial^2\mathcal{L}_1}{\partial\mathbf{w}^2}=\mathcal{O}(\frac{p^2}{n^2}).
    \nonumber
\end{equation}

With some algebras, we can also have
\begin{equation}
    \frac{\partial^2\mathcal{L}_2}{\partial\mathbf{w}^2}=\frac{1}{n}\mathbf{I}, \frac{\partial^2\mathcal{L}_3}{\partial\mathbf{w}^2}=\frac{1}{n^2}\mathbf{1}\mathbf{1}^\mathrm{T},
    \nonumber
\end{equation}

thus,
\begin{equation}
    \mathbf{H}_e=\mathcal{O}(\frac{p^2}{n^2})+\frac{\lambda_1}{n}\mathbf{I} + \frac{\lambda_2}{n^2}\mathbf{1}\mathbf{1}^\mathrm{T}=\frac{\lambda_1}{n}\mathbf{I} + \mathcal{O}(\frac{p^2+\lambda_2}{n^2}).
    \nonumber
\end{equation}
Therefore, if $\frac{\lambda_1}{n}\gg\frac{p^2+\lambda_2}{n^2}$, equivalent to $\lambda_1n\gg p^2+\lambda_2$, $\mathbf{H}_e$ is an almost diagonal matrix. Hence, $\mathbf{H}_e$ is positive definite~\cite{nakatsukasa2010absolute}. Then the function $\mathcal{F}(\mathbf{w})$ is convex on $\mathcal{C}=\{\mathbf{w}:|\mathbf{w}_i|\leq c\}$, and has unique optimal solution $\hat{\mathbf{w}}$.

\par Moreover, because $\mathcal{L}_1$ is our major decorrelation term, we hope $\mathcal{L}_1$ to dominate the terms $\lambda_1\mathcal{L}_2$ and  $\lambda_2\mathcal{L}_3$. On $\mathcal{C}$, we have $\mathcal{L}_1=\mathcal{O}(1)$, $\mathcal{L}_2=\mathcal{O}(1)$, and
$\alpha_i^2\alpha_k^2(\frac{1}{n}\sum_{i=1}^n\mathbf{H}_{i,j}\mathbf{H}_{i,k}\mathbf{w}_i-(\frac{1}{n}\sum_{i=1}^n\mathbf{H}_{i,j}\mathbf{w}_i)(\frac{1}{n}\sum_{i=1}^n\mathbf{H}_{i,k}\mathbf{w}_i))^2=\mathcal{O}(1)$. Thus $\mathcal{L}_1=\mathcal{O}(p^2)$. When $p^2\gg\max(\lambda_1, \lambda_2)$, $\mathcal{L}_1$ will dominate the regularization terms $\mathcal{L}_2$ and $\mathcal{L}_3$.
\end{proof}

\label{appendex::datasets}




\ifCLASSOPTIONcaptionsoff
  \newpage
\fi



%
\bibliographystyle{IEEEtran}
\bibliography{ref}

\begin{thebibliography}{10}
\providecommand{\url}[1]{#1}
\csname url@samestyle\endcsname
\providecommand{\newblock}{\relax}
\providecommand{\bibinfo}[2]{#2}
\providecommand{\BIBentrySTDinterwordspacing}{\spaceskip=0pt\relax}
\providecommand{\BIBentryALTinterwordstretchfactor}{4}
\providecommand{\BIBentryALTinterwordspacing}{\spaceskip=\fontdimen2\font plus
\BIBentryALTinterwordstretchfactor\fontdimen3\font minus
  \fontdimen4\font\relax}
\providecommand{\BIBforeignlanguage}[2]{{%
\expandafter\ifx\csname l@#1\endcsname\relax
\typeout{** WARNING: IEEEtran.bst: No hyphenation pattern has been}%
\typeout{** loaded for the language `#1'. Using the pattern for}%
\typeout{** the default language instead.}%
\else
\language=\csname l@#1\endcsname
\fi
#2}}
\providecommand{\BIBdecl}{\relax}
\BIBdecl

\bibitem{scarselli2008graph}
F.~Scarselli, M.~Gori, A.~C. Tsoi, M.~Hagenbuchner, and G.~Monfardini, ``The
  graph neural network model,'' \emph{IEEE Transactions on Neural Networks},
  vol.~20, no.~1, pp. 61--80, 2008.

\bibitem{kipf2016semi}
T.~N. Kipf and M.~Welling, ``Semi-supervised classification with graph
  convolutional networks,'' in \emph{ICLR}, 2016.

\bibitem{velivckovic2017graph}
P.~Veli{\v{c}}kovi{\'c}, G.~Cucurull, A.~Casanova, A.~Romero, P.~Lio, and
  Y.~Bengio, ``Graph attention networks,'' in \emph{ICLR}, 2017.

\bibitem{hamilton2017inductive}
W.~Hamilton, Z.~Ying, and J.~Leskovec, ``Inductive representation learning on
  large graphs,'' in \emph{NeurIPS}, 2017, pp. 1024--1034.

\bibitem{kuang2020stable}
K.~Kuang, R.~Xiong, P.~Cui, S.~Athey, and B.~Li, ``Stable prediction with model
  misspecification and agnostic distribution shift,'' in \emph{AAAI}, 2020.

\bibitem{shen2020stable}
Z.~Shen, P.~Cui, J.~Liu, T.~Zhang, B.~Li, and Z.~Chen, ``Stable learning via
  differentiated variable decorrelation,'' in \emph{KDD}, 2020, pp. 2185--2193.

\bibitem{sen2008collective}
P.~Sen, G.~Namata, M.~Bilgic, L.~Getoor, B.~Galligher, and T.~Eliassi-Rad,
  ``Collective classification in network data,'' \emph{AI magazine}, vol.~29,
  no.~3, pp. 93--93, 2008.

\bibitem{pmlr-v97-wu19e}
F.~Wu, A.~Souza, T.~Zhang, C.~Fifty, T.~Yu, and K.~Weinberger, ``Simplifying
  graph convolutional networks,'' in \emph{ICML}.\hskip 1em plus 0.5em minus
  0.4em\relax PMLR, 2019, pp. 6861--6871.

\bibitem{zadrozny2004learning}
B.~Zadrozny, ``Learning and evaluating classifiers under sample selection
  bias,'' in \emph{ICML}, 2004, p. 114.

\bibitem{yang2016revisiting}
Z.~Yang, W.~W. Cohen, and R.~Salakhutdinov, ``Revisiting semi-supervised
  learning with graph embeddings,'' in \emph{ICML}, 2016.

\bibitem{nurhonen1992property}
M.~Nurhonen and S.~Puntanen, ``A property of partitioned generalized
  regression,'' \emph{Communications in statistics-theory and methods},
  vol.~21, no.~6, pp. 1579--1583, 1992.

\bibitem{hainmueller2012entropy}
J.~Hainmueller, ``Entropy balancing for causal effects: A multivariate
  reweighting method to produce balanced samples in observational studies,''
  \emph{Political Analysis}, vol.~20, no.~1, pp. 25--46, 2012.

\bibitem{kuang2020causal}
K.~Kuang, L.~Li, Z.~Geng, L.~Xu, K.~Zhang, B.~Liao, H.~Huang, P.~Ding, W.~Miao,
  and Z.~Jiang, ``Causal inference,'' \emph{Engineering}, vol.~6, no.~3, pp.
  253--263, 2020.

\bibitem{kuang2018stable}
K.~Kuang, P.~Cui, S.~Athey, R.~Xiong, and B.~Li, ``Stable prediction across
  unknown environments,'' in \emph{SIGKDD}, 2018, pp. 1617--1626.

\bibitem{kuang2020data}
K.~Kuang, P.~Cui, H.~Zou, B.~Li, J.~Tao, F.~Wu, and S.~Yang, ``Data-driven
  variable decomposition for treatment effect estimation,'' \emph{TKDE}, 2020.

\bibitem{kreif2015evaluation}
N.~Kreif, R.~Grieve, I.~D{\'\i}az, and D.~Harrison, ``Evaluation of the effect
  of a continuous treatment: a machine learning approach with an application to
  treatment for traumatic brain injury,'' \emph{Health economics}, vol.~24,
  no.~9, pp. 1213--1228, 2015.

\bibitem{Bishop2006pattern}
C.~M. Bishop, \emph{Pattern recognition and machine learning}.\hskip 1em plus
  0.5em minus 0.4em\relax springer, 2006.

\bibitem{zou2020counterfactual}
H.~Zou, P.~Cui, B.~Li, Z.~Shen, J.~Ma, H.~Yang, and Y.~He, ``Counterfactual
  prediction for bundle treatment,'' in \emph{NeurIPS}, vol.~33, 2020.

\bibitem{carlson2010toward}
A.~Carlson, J.~Betteridge, B.~Kisiel, B.~Settles, E.~R. Hruschka, and T.~M.
  Mitchell, ``Toward an architecture for never-ending language learning,'' in
  \emph{Twenty-Fourth AAAI Conference on Artificial Intelligence}, 2010.

\bibitem{zhou2019graph}
F.~Zhou, T.~Li, H.~Zhou, H.~Zhu, and J.~Ye, ``Graph-based semi-supervised
  learning with non-ignorable non-response,'' in \emph{NeurIPS}, 2021.

\bibitem{klicpera_predict_2019}
J.~Klicpera, A.~Bojchevski, and S.~G{\"u}nnemann, ``Predict then propagate:
  Graph neural networks meet personalized pagerank,'' in \emph{ICLR}, 2019.

\bibitem{defferrard2016convolutional}
M.~Defferrard, X.~Bresson, and P.~Vandergheynst, ``Convolutional neural
  networks on graphs with fast localized spectral filtering,'' in
  \emph{NeurIPS}, 2016, pp. 3844--3852.

\bibitem{huang2006correcting}
J.~Huang, A.~Gretton, K.~Borgwardt, B.~Sch\"{o}lkopf, and A.~Smola,
  ``Correcting sample selection bias by unlabeled data,'' in \emph{NeurIPS},
  2007, pp. 1--8.

\bibitem{cogswell2015reducing}
M.~Cogswell, F.~Ahmed, R.~Girshick, L.~Zitnick, and D.~Batra, ``Reducing
  overfitting in deep networks by decorrelating representations,'' in
  \emph{ICLR}, 2016.

\bibitem{wang2020decorrelated}
X.~Wang, S.~Fan, K.~Kuang, C.~Shi, J.~Liu, and B.~Wang, ``Decorrelated
  clustering with data selection bias,'' in \emph{IJCAI}, 2020.

\bibitem{xu2018how}
\BIBentryALTinterwordspacing
K.~Xu, W.~Hu, J.~Leskovec, and S.~Jegelka, ``How powerful are graph neural
  networks?'' in \emph{ICLR}, 2019. [Online]. Available:
  \url{https://openreview.net/forum?id=ryGs6iA5Km}
\BIBentrySTDinterwordspacing

\bibitem{jin2021bite}
D.~Jin, X.~Song, Z.~Yu, Z.~Liu, H.~Zhang, Z.~Cheng, and J.~Han, ``Bite-gcn: A
  new gcn architecture via bidirectional convolution of topology and features
  on text-rich networks,'' in \emph{WSDM}, 2021, pp. 157--165.

\bibitem{fan2019metapath}
S.~Fan, J.~Zhu, X.~Han, C.~Shi, L.~Hu, B.~Ma, and Y.~Li, ``Metapath-guided
  heterogeneous graph neural network for intent recommendation,'' in
  \emph{Proceedings of the 25th ACM SIGKDD International Conference on
  Knowledge Discovery \& Data Mining}, 2019, pp. 2478--2486.

\bibitem{fan2020one2multi}
S.~Fan, X.~Wang, C.~Shi, E.~Lu, K.~Lin, and B.~Wang, ``One2multi graph
  autoencoder for multi-view graph clustering,'' in \emph{Proceedings of The
  Web Conference 2020}, 2020, pp. 3070--3076.

\bibitem{bai2020learning}
L.~Bai, L.~Cui, Y.~Jiao, L.~Rossi, and E.~Hancock, ``Learning backtrackless
  aligned-spatial graph convolutional networks for graph classification,''
  \emph{IEEE Transactions on Pattern Analysis and Machine Intelligence}, 2020.

\bibitem{cao2021knowledge}
Y.~Cao, H.~Peng, J.~Wu, Y.~Dou, J.~Li, and P.~S. Yu, ``Knowledge-preserving
  incremental social event detection via heterogeneous gnns,'' in \emph{WWW},
  2021, pp. 3383--3395.

\bibitem{jin2019graph}
D.~Jin, Z.~Liu, W.~Li, D.~He, and W.~Zhang, ``Graph convolutional networks meet
  markov random fields: Semi-supervised community detection in attribute
  networks,'' in \emph{AAAI}, vol.~33, no.~01, 2019, pp. 152--159.

\bibitem{lin2020exploratory}
X.~Lin, C.~Zhou, H.~Yang, J.~Wu, H.~Wang, Y.~Cao, and B.~Wang, ``Exploratory
  adversarial attacks on graph neural networks,'' in \emph{ICDM}.\hskip 1em
  plus 0.5em minus 0.4em\relax IEEE, 2020, pp. 1136--1141.

\bibitem{shen2020sample}
Z.~Shen, P.~Cui, T.~Zhang, and K.~Kuang, ``Stable learning via sample
  reweighting.'' in \emph{AAAI}, 2020, pp. 5692--5699.

\bibitem{ma2016decorrelation}
Z.~Ma, J.-H. Xue, A.~Leijon, Z.-H. Tan, Z.~Yang, and J.~Guo, ``Decorrelation of
  neutral vector variables: Theory and applications,'' \emph{IEEE transactions
  on neural networks and learning systems}, vol.~29, no.~1, pp. 129--143, 2016.

\bibitem{rodriguez2016regularizing}
P.~Rodr{\'\i}guez, J.~Gonzalez, G.~Cucurull, J.~M. Gonfaus, and X.~Roca,
  ``Regularizing cnns with locally constrained decorrelations,'' in
  \emph{ICLR}, 2016.

\bibitem{zhang2018removing}
Z.~Zhang, Y.~Zhang, and Z.~Li, ``Removing the feature correlation effect of
  multiplicative noise,'' in \emph{NeurIPS}, 2018.

\bibitem{nakatsukasa2010absolute}
Y.~Nakatsukasa, ``Absolute and relative weyl theorems for generalized
  eigenvalue problems,'' \emph{Linear Algebra and its Applications}, vol. 432,
  no.~1, pp. 242--248, 2010.

\end{thebibliography}

%

\begin{IEEEbiography}[{\includegraphics[width=1in,height=1.25in,clip,keepaspectratio]{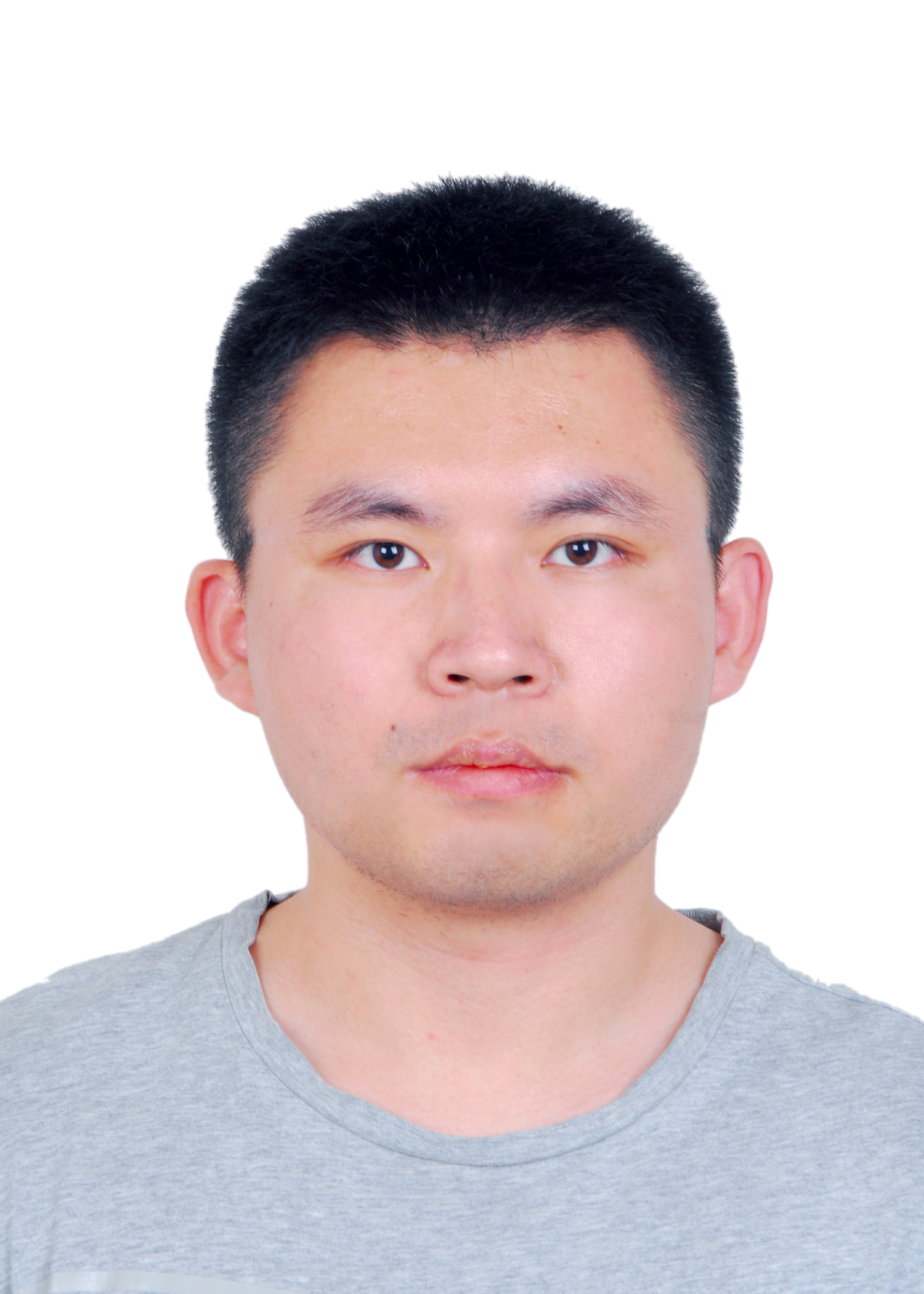}}]{Shaohua Fan}
received the B.E. degree in 2015 from Northeastern University and M.S. degree in 2018 from Beijing University of Posts and Telecommunications. He is a fourth-year
Ph.D. student in the Department of Computer
Science of Beijing University of Posts and Telecommunications and currently works as a visiting student at MILA. His main research interests including graph mining and causal machine learning. He has published several papers in major international conferences, including KDD, WWW, IJCAI, and CIKM etc.

\end{IEEEbiography}

\begin{IEEEbiography}[{\includegraphics[width=1in,height=1.25in,clip,keepaspectratio]{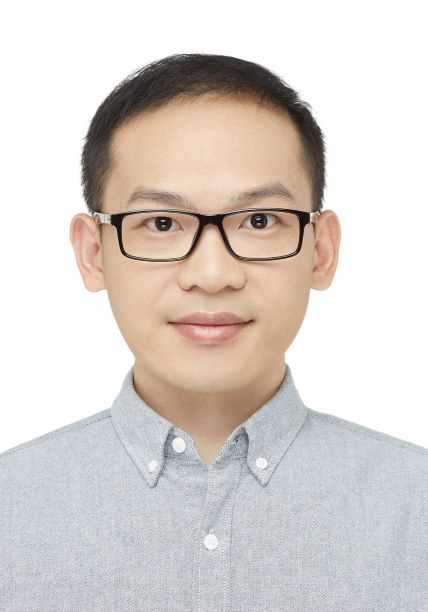}}]{Xiao Wang}
is an Associate Professor in the School of Computer Science, Beijing University of Posts and Telecommunications. He received his Ph.D. degree from the School of Computer Science and Technology, Tianjin University, Tianjin, China, in 2016. He was a postdoctoral researcher in Department of Computer Science and Technology, Tsinghua University, Beijing, China. His current research interests include data mining, social network analysis, and machine learning. Until now, he has published
more than 70 papers in refereed journals and conferences.
\end{IEEEbiography}


\begin{IEEEbiography}[{\includegraphics[width=1in,height=1.25in,clip,keepaspectratio]{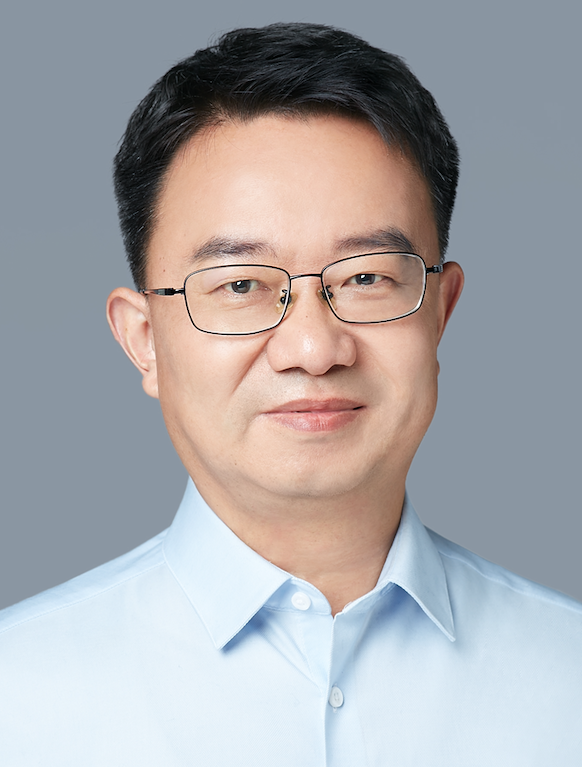}}]{Chuan Shi}
received the B.S. degree from the Jilin University in 2001, the M.S. degree from
the Wuhan University in 2004, and Ph.D. degree from the ICT of Chinese Academic of Sciences
in 2007. He is a professor and deputy director of Beijing Key Lab of Intelligent Telecommunications Software and Multimedia at present. His research interests are in data mining, machine learning, and evolutionary computing. He has published more than 100 papers in refereed journals and conferences.
\end{IEEEbiography}

\begin{IEEEbiography}[{\includegraphics[width=1in,height=1.25in,clip,keepaspectratio]{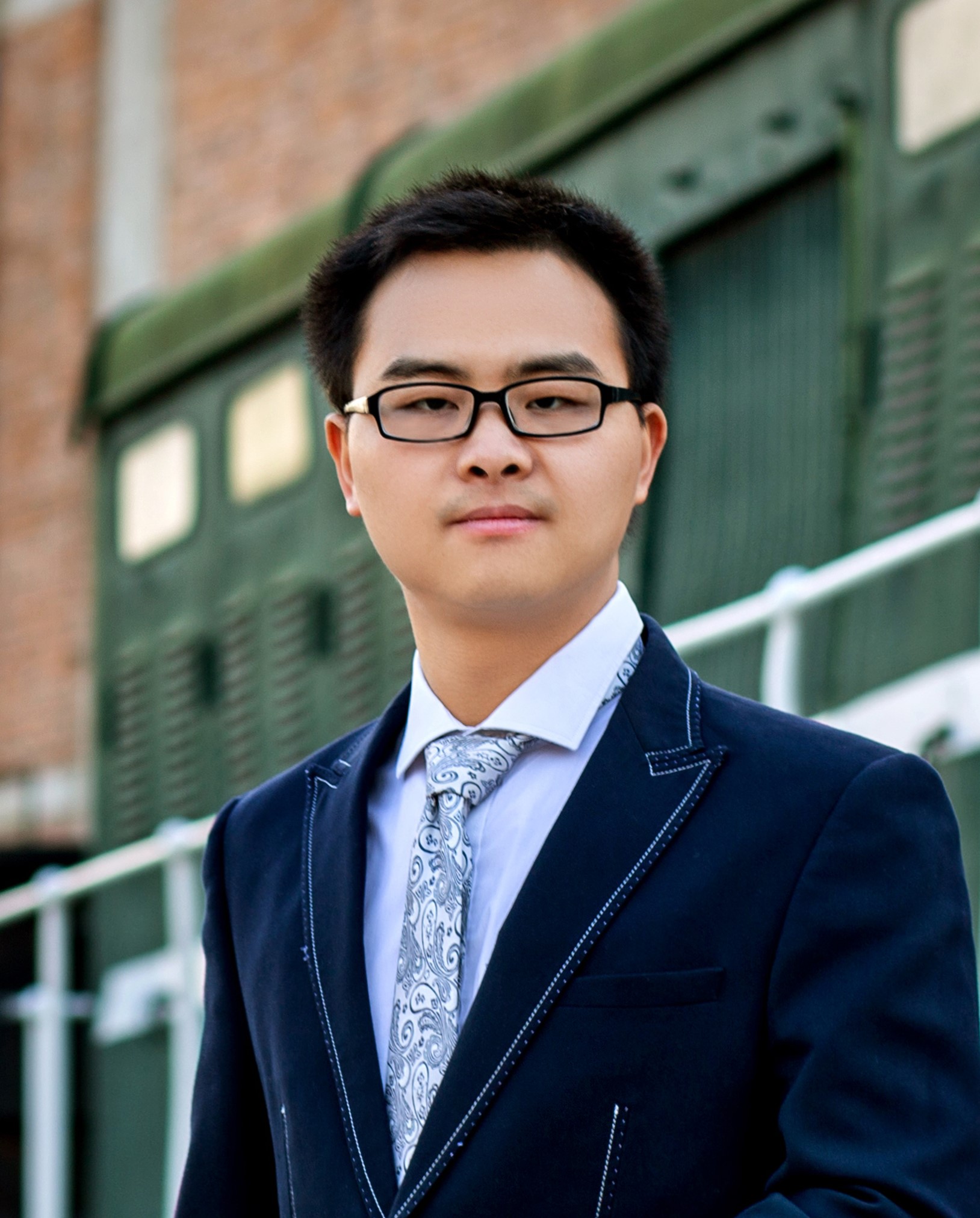}}]{Kun Kuang}
received his Ph.D. degree from Tsinghua University in 2019. He is now an Associate Professor in the College of Computer Science and Technology, Zhejiang University. He was a visiting scholar with Prof. Susan Athey's Group at Stanford University. His main research interests include Causal Inference, Artificial Intelligence, and Causally Regularized Machine Learning.  He has published over 40 papers in major international journals and conferences, including SIGKDD, ICML, ACM MM, AAAI, IJCAI, TKDE, TKDD, Engineering, and ICDM, etc.
\end{IEEEbiography}

\begin{IEEEbiography}[{\includegraphics[width=1in,height=1.25in,clip,keepaspectratio]{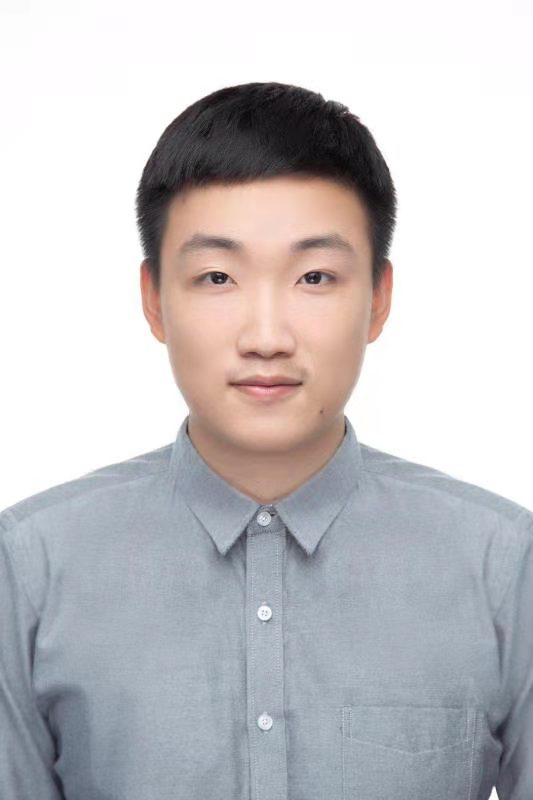}}]{Nian Liu}
received the B.E. degree in 2020 from Beijing University of Posts and Telecommunications. He is a second-year M.S. student in the Department of Computer Science of Beijing University of Posts and Telecommunications. His main research interests including graph mining and contrastive learning.
\end{IEEEbiography}

\begin{IEEEbiography}[{\includegraphics[width=1in,height=1.25in,clip,keepaspectratio]{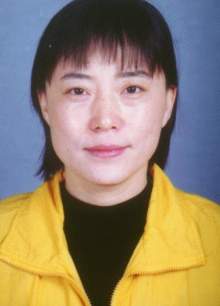}}]{Bai Wang}
received the B.S. degree from the Xian Jiaotong University, Xian, China and Ph.D. degree from the Beijing University of Posts and Telecommunications, Beijing, China. And she is currently a professor of computer science in BUPT. She was the director of Beijing Key Lab of Intelligent Telecommunications Software and Multimedia.
\end{IEEEbiography}




\end{document}